\documentclass[letterpaper]{article} 

\pdfoutput=1

\usepackage{aaai24}  
\usepackage{times}  
\usepackage{helvet}  
\usepackage{courier}  
\usepackage[hyphens]{url}  
\usepackage{graphicx} 
\usepackage{natbib}  
\usepackage{caption} 
\usepackage{algorithm}
\usepackage{algorithmic}

\usepackage{newfloat}
\usepackage{listings}
\DeclareCaptionStyle{ruled}{labelfont=normalfont,labelsep=colon,strut=off} 
\floatstyle{ruled}
\newfloat{listing}{tb}{lst}{}
\floatname{listing}{Listing}
\usepackage{tikz}
\usepackage{amssymb}
\usepackage{xcolor}
\usepackage{amsmath}
\usepackage{cleveref}
\usepackage{amsthm}
\usepackage{stmaryrd}
\usepackage{relsize}
\usepackage{booktabs}
\usepackage{subcaption}
\usepackage{thm-restate}

\newtheorem{definition}{Definition}
\newtheorem{example}{Example}
\newtheorem{lemma}{Lemma}
\newtheorem{proposition}{Proposition}

\DeclareMathOperator{\ltlN}{\normalfont\textsf{X}}
\DeclareMathOperator{\ltlG}{\normalfont\textsf{G}}
\DeclareMathOperator{\ltlF}{\normalfont\textsf{F}}
\DeclareMathOperator{\ltlU}{\normalfont\textsf{U}}

\makeatletter
\DeclareFontFamily{OMX}{MnSymbolE}{}
\DeclareSymbolFont{MnLargeSymbols}{OMX}{MnSymbolE}{m}{n}
\SetSymbolFont{MnLargeSymbols}{bold}{OMX}{MnSymbolE}{b}{n}
\DeclareFontShape{OMX}{MnSymbolE}{m}{n}{
	<-6>  MnSymbolE5
	<6-7>  MnSymbolE6
	<7-8>  MnSymbolE7
	<8-9>  MnSymbolE8
	<9-10> MnSymbolE9
	<10-12> MnSymbolE10
	<12->   MnSymbolE12
}{}
\DeclareFontShape{OMX}{MnSymbolE}{b}{n}{
	<-6>  MnSymbolE-Bold5
	<6-7>  MnSymbolE-Bold6
	<7-8>  MnSymbolE-Bold7
	<8-9>  MnSymbolE-Bold8
	<9-10> MnSymbolE-Bold9
	<10-12> MnSymbolE-Bold10
	<12->   MnSymbolE-Bold12
}{}

\let\llangle\@undefined
\let\rrangle\@undefined
\DeclareMathDelimiter{\llangle}{\mathopen}%
{MnLargeSymbols}{'164}{MnLargeSymbols}{'164}
\DeclareMathDelimiter{\rrangle}{\mathclose}%
{MnLargeSymbols}{'171}{MnLargeSymbols}{'171}
\makeatother

\definecolor{dkgreen}{rgb}{0,0.3,0}
\definecolor{dkblue}{rgb}{0,0.1,0.5}

\colorlet{comment-color}{black!50}

\lstdefinelanguage{custom-lang}{
	keywords={let, in, if, then, else, for, to, do, return, from, def},
	keywordstyle=[1]\bfseries,
	morekeywords=[2]{modelCheck},
	keywordstyle=[2]\itshape,
	morekeywords=[3]{product,nestedStateFormulas,LTLtoAPA},
	keywordstyle=[3]\itshape,
	comment=[l][\color{comment-color}]{//},
	literate=%
	{=}{{=}}1
	{//}{{\color{black!50!white}/\!/}}1
	{|}{{{|}}}1
	{:}{{{\textbf{:}}}}1
	{:=}{{{:=}}}1
	{@}{  }1
}

\lstdefinestyle{default}{
	escapeinside={(*}{*)},
	basicstyle=\fontsize{9}{10.8},
	columns=fullflexible,
	commentstyle=\color{black!50!white},
	framexleftmargin=1em,
	framexrightmargin=1ex,
	keepspaces=true,
	keywordstyle=,
	mathescape,
	numbers=left,
	numberblanklines=false,
	numbersep=0.5em,
	numberstyle=\relscale{0.75}\color{gray}\ttfamily,
	showstringspaces=true,
	stepnumber=1,
	xleftmargin=1.2em,
	xrightmargin=1em,
}

\lstnewenvironment{code}[1][]
{\small
	\lstset{
		style=default, 
		language=custom-lang,
		#1
	}
}
{}

\newcommand{\ap}{\mathit{AP}}
\newcommand{\pathVars}{\mathcal{V}}

\newcommand{\ldot}{\mathpunct{.}}

\newcommand{\nat}{\mathbb{N}}
\newcommand{\bool}{\mathbb{B}}

\newcommand{\calG}{\mathcal{G}}
\newcommand{\calA}{\mathcal{A}}

\newcommand{\calL}{\mathcal{L}}

\newcommand{\calO}{\mathcal{O}}

\newcommand{\bA}{{\boldsymbol{a}}}

\newcommand{\bF}{{\boldsymbol{f}}}

\newcommand{\EXPTIME}{\texttt{EXPTIME}}
\newcommand{\ATLS}{ATL$^*$}

\newcommand{\CTLS}{CTL$^*$}
\newcommand{\ATL}{ATL}
\newcommand{\CTL}{CTL}

\newcommand{\HyperATLS}{HyperATL$^*$}
\newcommand{\HyperATLSS}{HyperATL$^*_S$}
\newcommand{\HyperLTL}{HyperLTL}
\newcommand{\HyperCTLS}{HyperCTL$^*$}

\newcommand{\tool}{\texttt{HyMASMC}}
\newcommand{\mcmas}{\texttt{MCMAS}}
\newcommand{\mcmassl}{\texttt{MCMAS-SL[1G]}}

\newcommand{\agent}[1]{#1}

\newcommand{\atlsToHyper}[1]{\llparenthesis #1 \rrparenthesis}

\newcommand{\agents}{\mathit{Agts}}
\newcommand{\moves}{\mathbb{A}}

\newcommand{\play}[0]{\mathit{Play}}
\newcommand{\strats}[1]{\mathit{Str}(#1)}

\newcommand{\lQ}{\langle\hspace{-0.6mm}[}
\newcommand{\rQ}{]\hspace{-0.6mm}\rangle}

\definecolor{myred}{HTML}{B06161}
\definecolor{mygreen}{HTML}{508D69}

\newif\iffullversion
\fullversiontrue

\newcommand{\ifFull}[2]{\iffullversion#1\else#2\fi}

\title{On Alternating-Time Temporal Logic, Hyperproperties, and Strategy Sharing}
\author{
    Raven Beutner, Bernd Finkbeiner
}
\affiliations{
	CISPA Helmholtz Center for Information Security, Germany
}

\begin{document}

\maketitle

\begin{abstract}
Alternating-time temporal logic (\ATLS{}) is a well-established framework for formal reasoning about multi-agent systems. 
However, while \ATLS{} can reason about the strategic ability of agents (e.g., some coalition $A$ can ensure that a goal is reached eventually), we cannot \emph{compare} multiple strategic interactions, nor can we require multiple agents to follow the \emph{same} strategy. 
For example, we cannot state that coalition $A$ can reach a goal \emph{sooner} (or \emph{more often}) than some other coalition $A'$. 
In this paper, we propose \HyperATLSS{}, an extension of \ATLS{} in which we can \textbf{(1)} compare the outcome of multiple strategic interactions w.r.t.~a \emph{hyperproperty}, i.e., a property that refers to multiple paths at the same time, and \textbf{(2)} enforce that some agents \emph{share} the same strategy.
We show that \HyperATLSS{} is a rich specification language that captures important AI-related properties that were out of reach of existing logics. 
We prove that model checking of \HyperATLSS{} on concurrent game structures is decidable.
We implement our model-checking algorithm in a tool we call \tool{} and evaluate it on a range of benchmarks.
\end{abstract}

\section{Introduction}\label{sec:intro}

Logics play a key role in the specification and verification of strategic properties in multi-agent systems (MAS) \cite{CalegariCMO21}.
One of the most influential temporal logics for MASs is alternating-time temporal logic (\ATLS), which extends \CTLS{}  with (implicit) quantification over strategies \cite{AlurHK02}. 
As an example, assume we want to formally verify that a set of agents $A$ can ensure that some temporal objective $\psi$ is ultimately fulfilled.
We can express this as the \ATLS{} formula $\llangle A \rrangle \ltlF \psi$, stating that the agents in $A$ have a joint strategy that ensures that all compatible executions eventually ($\ltlF$) satisfy $\psi$.
Likewise, we can express that coalition $A$ has \emph{no} strategy to ensure that $\psi$ is reached as $\llbracket A \rrbracket \ltlG \neg \psi$, i.e., for every strategy of $A$, some execution globally ($\ltlG$) satisfies $\neg \psi$.

However, in many situations, we are interested not only in the strategic (in)ability of a coalition but also in comparing the ability of multiple coalitions. 
For example, we might ask if some coalition $A$ is able to reach some goal $\psi$ strictly sooner (or more often) than some other coalition $A'$. 
Indeed, important game-theoretic concepts such as \emph{Shapley values} \cite{shapley1953value} are inherently based on the relative contribution of individual agents:
To compute the Shapley value for some agent $\agent{i}$, we need to \emph{compare} the ability of some arbitrary coalitions $A$ with that of $A \cup \{\agent{i}\}$.
Stating such comparison-based properties in \ATLS{} is impossible as \ATLS{} only considers a single path in isolation.

\paragraph{Hyperproperties.}

In contrast, the formal methods community has extensively studied properties that relate \emph{multiple} system executions and coined them \emph{hyperproperties} \cite{ClarksonS08}. 
In this paper, we bring the powerful concept of hyperproperties to the realm of AI and MASs.
We introduce \HyperATLSS{} -- a temporal logic that combines \emph{strategic reasoning} (as found in \ATLS{}), the ability to compare executions w.r.t.~a \emph{hyperproperty} (as, e.g.,  found in \HyperATLS{}), and the possibility of enforcing agents to \emph{share} strategies. 
As in \HyperATLS{} \cite{BeutnerF21,BeutnerF23}, we bind the outcome of a strategic interaction (resulting from an \ATLS{}-like quantification) to a \emph{path variable} and can then refer to atomic propositions on multiple paths. 
This combination of strategic reasoning and hyperproperties is needed for many AI-related properties; not only the information-flow properties envisioned in \cite{BeutnerF21,BeutnerF23}.
For example, \HyperATLSS{} allows us to express that coalition $A$ can reach $\psi$ strictly sooner than coalition $A'$ as follows.
\begin{align*}
	\llangle A \rrangle \pi\ldot \llbracket A' \rrbracket \pi'\ldot (\neg \psi_{\pi'}) \ltlU (\neg \psi_{\pi'} \land \psi_{\pi}).
\end{align*}
This formula states that there exist strategies for the agents in $A$, such that for every path $\pi$ under those strategies, it holds that: under \emph{every} strategy for the agents in $A'$, there exists some compatible path $\pi'$, such that $\pi$ reaches $\psi$ (denoted $\psi_{\pi}$) before $\pi'$ does (expressed using LTL's \emph{until} operator $\ltlU$).
Phrased differently, some strategy $A$ can ensure that $\psi$ is reached \emph{strictly} faster than any strategy for $A'$ could.

Note that this approach is very flexible, as we can compare $\pi$ and $\pi'$ w.r.t. to an arbitrary temporal property (e.g., $\pi$ reaches $\psi$ \emph{more often} than $\pi'$). 
This goes well beyond the capabilities of \ATLS{}, even when extended with quantitative operators (cf.~\Cref{sec:related-work}).

\paragraph{Strategy Sharing and \HyperATLSS{}.}

\HyperATLSS{} then extends \HyperATLS{} with the ability to force agents to follow the same strategy.
A \emph{sharing constraint} $\xi$ is a set of pairs of agents, and the \HyperATLSS{} formula $\llangle A \rrangle_\xi \, \pi\ldot \varphi$ requires that coalition $A$ can satisfy $\varphi$, under the assumption that all agents $(\agent{i}, \agent{j}) \in \xi$ play the \emph{same} strategy; similar to what is possible in strategy logic \cite{MogaveroMPV14,ChatterjeeHP10} in a \emph{non-hyper} setting.

\begin{example}\label{ex:sharing}
	Assume we deal with a MAS modeling a planning task with multiple robots and want to ensure that robots in coalition $A$ can reach some target state. 
	To keep the employment overhead as small as possible, we might ask if the robots can follow some optimal trajectory (i.e., reach the target as fast as possible), despite all using the \emph{same} strategy.
	We can express this in \HyperATLSS{} as follows
	\begin{align*}
		\llangle A \rrangle_{\{(\agent{i}, \agent{j}) \mid \agent{i}, \agent{j} \in A\}} \, \pi \ldot \llbracket A \rrbracket \, \pi'\ldot  (\neg \mathit{target}_{\pi'}) \ltlU  \mathit{target}_{\pi}
	\end{align*} 
	stating that \emph{all} robots in $A$ can use a \emph{shared} strategy (on path $\pi$) that reaches the target at least as fast as they can without the constraint that they must play the same strategy (path $\pi'$).
	Such shareable strategies are, e.g., key for scalable synthesis \cite{AttieE98}.
\end{example}

We provide further \HyperATLSS{} examples (such as determinism and good-enough synthesis)  in \Cref{sec:sub:eval-hyper}.

\paragraph{Model Checking.}

We show that model checking (MC) of \HyperATLSS{} on finite-state concurrent game structures (a standard model of MASs) is decidable.
As \HyperATLSS{} can relate multiple computation paths, we cannot employ the tree-automaton-based MC approach for \ATLS{} \cite{AlurHK02}. 
Instead, we develop a MC algorithm based on alternating word automata. 
Our algorithm iteratively simulates path quantification within an automaton, while ensuring that the strategy-sharing constraints between agents are fulfilled.  

\paragraph{Implementation.}

We implement our model-checking algorithm (for \emph{full} \HyperATLSS{}) in a tool we call \tool{}.
Using \tool{}, we can, for the first time, automatically check properties beyond the self-composition fragment of \HyperATLS{} -- the largest fragment supported by previous tools \cite{BeutnerF21,BeutnerF23} (cf.~\Cref{sec:related-work}).
We evaluate \tool{} by verifying a range of properties in MASs from the literature.
Our experiments show that our algorithm performs well on \emph{non}-hyper instances that could already be handled using existing solvers \cite{CermakLM15} and can successfully verify hyperproperties that cannot be expressed in any existing logic, let alone checked with any existing tool. 

\paragraph{Supplementary Material.}

Detailed proofs and additional material can be found in \ifFull{the appendix}{\cite{fullVersion}}.

\section{Related Work}\label{sec:related-work}

Various works have extended \ATLS{} with abilities to reason about probabilistic systems \cite{ChenL07a}, incomplete information \cite{BelardinelliLMR17,BerthonMM17,BelardinelliLM19}, and finite traces \cite{BelardinelliLMR18}.
All of these extension refer to individual paths and cannot express properties that relate \emph{multiple} paths.
While resource-aware extensions offer \emph{quantitative} reasoning \cite{AlechinaDL20,BouyerKMMMP19,HenzingerP06,JamrogaKP16,ChenL07a}, they still cannot state properties that go beyond computing quantities on individual paths.
Strategy logic (SL) treats strategies as first-class objects and can naturally express properties where some agents share the same strategy \cite{MogaveroMPV14,ChatterjeeHP10}.
While SL can compare the same strategy in different scenarios, it is limited to a \emph{boolean} combination of LTL properties on individual paths, i.e., we cannot compare different paths w.r.t.~a \emph{temporal} hyperproperty.
All properties we consider in \Cref{sec:intro,sec:sub:eval-hyper} cannot be expressed in SL.
Most existing hyperlogics, including \HyperLTL{} and \HyperCTLS{} \cite{ClarksonFKMRS14}, reason about paths in a (non-strategic) transition system.
\HyperATLS{} was the first temporal logic that combined \emph{strategic} reasoning with the ability to express hyperproperties \cite{BeutnerF21,BeutnerF23}.
This captures strategic \emph{information-flow policies} such as simulation-based non-interference \cite{MantelS01} and non-deducibility of strategies \cite{WittboldJ90}.
\HyperATLSS{} extends \HyperATLS{} with the ability to force agents to share the same strategy, which is useful for many AI-related properties (cf.~\Cref{ex:sharing}).
Moreover, automated verification of \HyperATLS{} was, so far, only possible for the self-composition fragment \cite{BeutnerF23}. 
In this fragment, all quantifiers are grouped together by constructing the self-composition of a MAS \cite{BartheDR11}, which reduces verification to a parity game.
While this fragment suffices for many security-related properties (which are naturally defined in terms of a self-composition), it does not capture any of the properties discussed in \Cref{sec:intro,sec:eval}.
In contrast, our model-checking algorithm (implemented in \tool{}) uses iterative quantifier elimination and is applicable to \emph{all} \HyperATLSS{} formulas.
In terms of tool support, the \mcmas{} tool family \cite{LomuscioQR09} implements a range of model checkers for strategic properties (e.g., specified in \ATLS{} or SL), often with a strong focus on knowledge \cite{FHMV1995,HoekW03a}.
Generally, knowledge properties \emph{are} hyperproperties; to ``know something'' means that it should hold on all indistinguishable paths, effectively relating multiple paths in a system \cite{BozzelliMP15,BeutnerFFM23}.
However, before \tool{}, none of the existing verifiers could check general hyperproperties (beyond knowledge) in MASs.

\section{Preliminaries}

For two functions $f : X \to Z$ and $f' : Y \to Z$ with $X \cap Y = \emptyset$, we define $f \oplus f' : X \cup Y \to Z$ as the union of both functions.
We let $\ap$ be a fixed finite set of atomic propositions and let $\agents$ be a fixed finite set of agents.
For a set of agent $A \subseteq \agents$, we define $\overline{A} := \agents \setminus A$.
Given some set $X$, we write $X^+$ (resp.~$X^\omega$) for the set of non-empty finite (resp.~infinite) sequences over $X$.
For $u \in X^\omega$ and $k \in \nat$, we write $x(k)$ for the $k$th element, $u[k, \infty]$ for the infinite suffix starting at position $k$, and $u[0, k]$ for the finite prefix up to $k$. 
As the underlying model of MASs, we use concurrent game structures (CGS).

\begin{definition}[\citet{AlurHK02}]
	A concurrent game structure is a tuple $\calG = (S, s_0, \moves, \kappa, L)$ where $S$ is a finite set of states, $s_0 \in S$ is an initial state, $\moves$ is a finite set of actions, $\kappa : S \times (\agents \to \moves) \to S$ is a transition function, and $L : S \to 2^\ap$ is a state labeling.
\end{definition}

An \emph{action vector} is a function $\bA : \agents \to \moves$ assigning an action to each agent. 
Given a state $s$ and action vector $\bA$, the transition function $\kappa$ determines the next state $\kappa(s, \bA)$.
A \emph{strategy} in $\calG$ is a function $f : S^+ \to \moves$, mapping finite paths to actions.
We denote the set of all strategies in $\calG$ with $\strats{\calG}$.
Given a state $s \in S$ and \emph{strategy vector} $\bF : \agents \to \strats{\calG}$ mapping each agent to a strategy, we can construct the path $\play_\calG(s, \bF) \in S^\omega$ that results from each agent acting according to the strategy defined by $\bF$.
Formally, we define $\play_\calG(s, \bF)$ as the unique infinite path $p \in S^\omega$ such that $p(0) = s$, and for every $k \in \nat$ we have $p(k+1) = \kappa\big(p(k), \bA_k\big)$ where $\bA_k$ is the action vector defined by $\bA_k(\agent{i}) := \bF(\agent{i})(p[0,k])$ for $\agent{i} \in \agents$.
That is, we map each agent $\agent{i}$ to the action selected by strategy $\bF(\agent{i})$ on the prefix $p[0,k]$, and update the state according to $\kappa$.

Note that our CGS definition does not include a protocol function $\varrho : S \times \agents \to (2^\moves \setminus \{\emptyset\})$ that, in each state, assigns each agent a set of allowed actions.
We can simulate the protocol $\varrho$ in the transition function $\kappa$ by ``rerouting'' every action that is invalid (according to $\varrho$) to some allowed action, effectively limiting the available actions of an agent.

\paragraph{\ATLS{}.}

We briefly recall the syntax and semantics of \ATLS{}.
Path and state formulas in \ATLS{} are defined as follows:
\begin{align*}
	\psi &:= a \mid \psi \land \psi \mid \neg \psi \mid \ltlN \psi  \mid \psi \ltlU \psi \mid \varphi\\
	\varphi &:= \llangle A \rrangle \, \psi \mid \llbracket A \rrbracket \, \psi
\end{align*}
where $a \in \ap$ and $A \subseteq \agents$.
The temporal $\ltlN$ refers to the \emph{next} timepoint, and $\psi_1 \ltlU \psi_2$ states that $\psi_2$ holds at some future timestep and $\psi_1$ holds at all timesteps \emph{until} then.
We use the standard Boolean connectives $\lor, \to, \leftrightarrow$, and Boolean constants $\top, \bot$, as well as the derived temporal operators \emph{eventually} $\ltlF \psi := \top \ltlU \psi$ and \emph{globally} $\ltlG \psi :=\neg \ltlF \neg \psi$.
For a path $p \in S^\omega$, we evaluate a path formula as expected:
\begin{align*}
	p &\models_\calG a &\text{iff } \quad&a \in L\big(p(0)\big)\\
	p &\models_\calG \psi_1 \land \psi_2 &\text{iff } \quad &p \models_\calG \psi_1 \text{ and } p \models_\calG \psi_2\\
	p &\models_\calG \neg \psi &\text{iff } \quad &p \not\models_\calG \psi\\
	p &\models_\calG \ltlN \psi &\text{iff } \quad &p[1, \infty] \models_\calG \psi\\
	p &\models_\calG \psi_1 \ltlU \psi_2 &\text{iff } \quad &\exists k \in \nat\ldot p[k, \infty] \models_\calG \psi_2 \text{ and } \\
	& \quad\quad\quad\quad\quad\quad\quad\quad\quad\quad \forall 0 \leq m < k\ldot p[m, \infty] \models_\calG \psi_1 \span \span\\
	p &\models_\calG \varphi &\text{iff } \quad&p(0) \models_\calG \varphi
\end{align*}
For a state $s \in S$, we define:
\begin{align*}
	s &\models_\calG \llangle A \rrangle \psi \quad \text{iff } \quad \exists \bF : A \to \strats{\calG}\ldot \\
	&\quad\quad\forall \bF' : \overline{A} \to \strats{\calG}\ldot \play_\calG\big(s, \bF \oplus \bF'  \big)] \models_\calG \psi\\
	s &\models_\calG \llbracket A \rrbracket \psi \quad \text{iff } \quad \forall \bF : A \to \strats{\calG}\ldot \\
	&\quad\quad\exists \bF' : \overline{A}  \to \strats{\calG}\ldot \play_\calG\big( s, \bF \oplus \bF'  \big)] \models_\calG \psi.
\end{align*}
That is, $\llangle A \rrangle \psi$ holds in state $s$ if the agents in $A$ can enforce $\psi$.
Formally, this means that there exists a strategy for each agent in $A$ (formalized as function $\bF$) such that -- no matter what strategy the agents in $\overline{A} = \agents \setminus A$ follow (function $\bF'$) -- the resulting path satisfies path formula $\psi$.
Conversely, $\llbracket A \rrbracket \psi$ states that coalition $A$ cannot \emph{avoid} $\psi$, i.e., every strategy for $A$ admits some path that satisfies $\psi$.

A CGS $\calG = (S, s_0, \moves, \kappa, L)$ satisfies $\varphi$, written $\calG \models_{\text{\ATLS{}}} \varphi$, if $s_0 \models_\calG \varphi$, i.e., $\varphi$ holds in the initial state.

\section{HyperATL$^*_S$}

In \ATLS{}, we can quantify over paths in the system (constructed by some strategy), but with each nested quantification, we create a new path, effectively losing the handle of the path(s) constructed previously.
Consequently, formula $\llangle A \rrangle \llangle A' \rrangle \psi$ is equivalent to $\llangle A' \rrangle \psi$. 
In \HyperATLSS{}, we want to explicitly state hyperproperties on \emph{multiple} paths.
To accomplish this, we extend \ATLS{} with the notation of \emph{path variables} and -- whenever we encounter a strategic path quantifier -- bind the outcomes of this quantification to such a variable, similar to \HyperCTLS{} \cite{ClarksonFKMRS14} and \HyperATLS{} \cite{BeutnerF21,BeutnerF23}.

\paragraph{Syntax.}

Let $\pathVars = \{\pi, \pi', \ldots\}$ be a set of \emph{path variables}.
Path and state formulas in \HyperATLSS{} are generated by the following grammar.
\begin{align*}
	\psi &:= a_\pi \mid \psi \land \psi \mid \neg \psi \mid \ltlN \psi \mid \psi \ltlU \psi  \mid \varphi_\pi\\
	\varphi &:= \llangle A \rrangle_\xi \, \pi \ldot \varphi \mid \llbracket A \rrbracket_\xi \, \pi \ldot \varphi \mid \psi
\end{align*}
where $a \in \ap$, $\pi \in \pathVars$, $A \subseteq \agents$, and $\xi \subseteq \agents \times \agents$ is a \emph{sharing constraint}.
We assume that nested state formulas are closed, i.e., for each atomic formula $a_\pi$, path variable $\pi$ is bound by some quantifier. 

Similar to \ATLS{}, formula $\llangle A \rrangle_\xi \, \pi \ldot \varphi$ states that there exists a strategy for coalition $A$ such that all paths under that strategy satisfy $\varphi$. 
However, differently from \ATLS{}, we bind this path to the path variable $\pi$.
We can then use path variables to refer to multiple paths via indexed atomic propositions.
The constraint $\xi$ poses restrictions on the agents' strategies: if $(\agent{i}, \agent{j}) \in \xi$, then agents $\agent{i}$ and $\agent{j}$ should play the same strategy.
We assume that for each quantifier $\llangle A \rrangle_\xi $ and $\llbracket A \rrbracket_\xi$, the sharing constraint satisfies $\xi \subseteq (A \times A) \cup (\overline{A} \times \overline{A})$, i.e., $\xi$ can enforce strategy sharing between agents in $A$ and between agents in $\overline{A}$.  
We omit $\xi$ if $\xi = \emptyset$.

\paragraph{Semantics.}

We evaluate \HyperATLSS{} formulas in the context of a \emph{path assignment}, which is a partial mapping $\Pi : \pathVars \rightharpoonup S^\omega$.
We write $\emptyset$ for the path assignment with an empty domain.
Given $k \in \nat$, we define $\Pi[k, \infty]$ as the assignment defined by $\Pi[k, \infty](\pi) := \Pi(\pi)[k, \infty]$, i.e., the assignment where all paths are (synchronously) shifted by $k$ positions. 
For a path $p \in S^\omega$, we define $\Pi[\pi \mapsto p]$ as the updated assignment that maps $\pi$ to $p$.
For path formulas, we define
\begin{align*}
	\Pi &\models_\calG a_\pi &\text{iff } \quad&a \in L\big(\Pi(\pi)(0)\big)\\
	\Pi &\models_\calG \psi_1 \land \psi_2 &\text{iff } \quad &\Pi \models_\calG \psi_1 \text{ and } \Pi \models_\calG \psi_2\\
	\Pi &\models_\calG \neg \psi &\text{iff } \quad &\Pi \not\models_\calG \psi\\
	\Pi &\models_\calG \ltlN \psi &\text{iff } \quad &\Pi[1, \infty] \models_\calG \psi\\
	\Pi &\models_\calG \psi_1 \ltlU \psi_2 &\text{iff } \quad &\exists k \in \nat\ldot \Pi[k, \infty] \models_\calG \psi_2 \text{ and } \\
	&  \quad\quad\quad\quad\quad\quad\quad\quad\quad\quad \forall 0 \leq m < k\ldot \Pi[m, \infty] \models_\calG \psi_1 \span \span\\
	\Pi &\models_\calG \varphi_\pi &\text{iff } \quad&\Pi(\pi)(0), \emptyset \models \varphi.
\end{align*}%
Whenever we check if $a_\pi$ currently holds, we check if $a$ holds on the path that is bound to $\pi$.
A nested state formula $\varphi_\pi$ holds iff $\varphi$ holds in the first state of the path bound to $\pi$.

Given a set of agents $A \subseteq \agents$ and sharing constraints $\xi$, we define $\mathit{shr}_\calG(A, \xi) := \{ \bF : A \to \strats{\calG} \mid \forall \agent{i}, \agent{j} \in A\ldot (\agent{i}, \agent{j}) \in \xi \Rightarrow \bF(\agent{i}) = \bF(\agent{j}) \}$, i.e., all strategy vectors for $A$ that satisfy the constraints in $\xi$. 
\HyperATLSS{} state formulas are evaluated in a state $s$ and path assignments $\Pi$. 
For each strategy quantifier, we construct a new path and bind this path to a path variable in $\Pi$:
\begin{align*}
	s, \Pi &\models_\calG \psi  &\text{iff } \quad &\Pi \models_\calG \psi\\
	s, \Pi &\models_\calG \llangle A \rrangle_\xi \, \pi \ldot \varphi  &\text{iff } \quad &\exists \bF \in \mathit{shr}_\calG(A, \xi) \ldot \\
	&\forall \bF' \! \in \! \mathit{shr}_\calG(\overline{A}, \xi) \ldot s, \Pi[\pi \!\mapsto \!\play_\calG\big(s, \bF \oplus \bF'  \big)] \models_\calG \varphi  \span\span \\
	s, \Pi &\models_\calG \llbracket A \rrbracket_\xi \, \pi \ldot \varphi  &\text{iff } \quad &\forall \bF \in \mathit{shr}_\calG(A, \xi) \ldot \\
	&\exists \bF' \! \in \! \mathit{shr}_\calG(\overline{A}, \xi) \ldot s, \Pi[\pi \!\mapsto \!\play_\calG\big(s, \bF \oplus \bF'  \big)] \models_\calG \varphi  \span\span
\end{align*}
Take $\llangle A \rrangle_\xi \, \pi\ldot \varphi $ as an example.
As in \ATLS{}, we existentially quantify over strategies for the agents in $A$ (subject to the condition that they respect the sharing constraints in $\xi$), followed by universal quantification over strategies for agents in $\overline{A}$ (again, subject to $\xi$).
The resulting strategy vector $\bF \oplus \bF'$ then yields a unique path $\play_\calG\big(s, \bF \oplus \bF'  \big)$, which we bind to path variable $\pi$ and continue evaluation of $\varphi$.
Note that in case $\xi = \emptyset$, the quantification behavior is very close to that of \ATLS{} as $\mathit{shr}_\calG(A, \emptyset)$ contains all functions $A \to \strats{\calG}$.
The important difference to \ATLS{} is that once we have constructed the path $\play_\calG\big(s, \bF \oplus \bF'  \big)$, we do not immediately evaluate a path formula  but rather add the path to our current assignment. 
Without sharing constraints, \HyperATLSS{} corresponds to \HyperATLS{} \cite{BeutnerF21,BeutnerF23} and is strictly more expressive than \ATLS{}.

We say that $\calG$ satisfies $\varphi$, written $\calG \models \varphi$, if $s_0, \emptyset \models_\calG \varphi$.

\begin{figure}[!t]
	\centering
	\begin{tikzpicture}
		\node[circle, inner sep=2pt,draw, thick, label=below:{\small $\emptyset$}] at (0,0) (n0) {\small $s_0$};

		\node[circle, inner sep=2pt, draw, thick, label=below:{\small $\emptyset$}] at (4,0) (n1) {\small $s_1$};

		\node[circle, inner sep=2pt, draw, thick, label=below:{\small $\{w\}$}] at (2,0) (n2) {\small $s_2$};

		\coordinate[] (c1) at (0, 0.7);
		\coordinate[] (c2) at (4, 0.7);
		
		\draw[-, thick] (n0) to (c1);
		\draw[-, thick] (c1) to node[above,align=center] {\scriptsize $(\texttt{g}, \texttt{r}, \texttt{nr})$, $(\texttt{g}, \texttt{nr}, \texttt{r})$} (c2);
		\draw[->, thick] (c2) to (n1);
		
		\draw[->, thick] (n1) to node[below] {\scriptsize$(\_, \_, \_)$} (n2);
		
		\draw[->, thick] (n2) to[bend left] node[below] {\scriptsize$(\_, \_, \_)$} (n0);
		
		\draw[->, thick] (-0.4, 0.4) to[] (n0);
		
		\draw[->, thick] (n0) to[] node[above, sloped] {\scriptsize$(\texttt{g}, \texttt{r}, \texttt{r})$} (n2);
		
		\draw[->, thick] (n0) to [loop left] node[left,align=center] {\scriptsize $(\texttt{ng}, \_, \_)$\\\scriptsize $(\texttt{g}, \texttt{nr}, \texttt{nr})$} (n0);
		
	\end{tikzpicture}

	\vspace{-1mm}

	\caption{A simple CGS with $\agents = \{\agent{\mathit{sched}}, \agent{\mathit{W1}}, \agent{\mathit{W2}}\}$. Each edge has the form $(a_1, a_2, a_3)$ where $a_1, a_2$, and $a_3$ are the actions of $\agent{\mathit{sched}}$, $\agent{\mathit{W1}}$, and $\agent{\mathit{W2}}$, respectively. We write ``$\_$'' for an arbitrary action.}\label{fig:cgs}
\end{figure}

\begin{example}[Running Example]\label{ex:cgs}
	Let us consider a very simple CGS between agents $\agents = \{\agent{\mathit{sched}}, \agent{\mathit{W1}}, \agent{\mathit{W2}}\}$, describing a scheduler and  two worker agents.
	The scheduler $\agent{\mathit{sched}}$ can choose actions $\{\texttt{g}, \texttt{ng}\}$ modeling a \textbf{g}rant or \textbf{n}o \textbf{g}rant, and each of the workers can choose actions $\{\texttt{r}, \texttt{nr}\}$ modeling a \textbf{r}equest to work or \textbf{n}o \textbf{r}equest to work.
	We model the dynamics of the CGS in \Cref{fig:cgs}.
	If the scheduler chooses $\texttt{ng}$ or both of the workers do not request to work, we remain in idle state $s_0$.
	If the scheduler grants work and \emph{both} workers request to work, we directly transition to the working state $s_2$ where proposition $w \in \ap$ holds. 
	If only one of the workers requests work, we also transition to $s_2$ but pass through $s_1$, i.e., the work is delayed by one step. 
	
	Let us assume we want to verify that coalition $\{\agent{\mathit{sched}}, \agent{\mathit{W1}}, \agent{\mathit{W2}}\}$ can reach the work state $s_2$ (strictly) faster than $\{\agent{\mathit{sched}}, \agent{\mathit{W1}}\}$.
	As argued in the introduction, we can express this using the following \HyperATLSS{} formula \\[-2mm]
	\scalebox{0.95}{\parbox{\linewidth}{
	\begin{align*}
		\llangle \agent{\mathit{sched}}, \agent{\mathit{W1}}, \agent{\mathit{W2}} \rrangle \, \pi \ldot \llbracket \agent{\mathit{sched}}, \agent{\mathit{W1}} \rrbracket \, \pi'\ldot (\neg \mathit{w}_{\pi'}) \ltlU (\neg \mathit{w}_{\pi'} \land \mathit{w}_{\pi}).
	\end{align*} 
	}}\\[-1mm]
	This formula holds in the above CGS: $\{\agent{\mathit{sched}}, \agent{\mathit{W1}}, \agent{\mathit{W2}}\}$ can construct a path $\pi$ where $w$ holds in the second step, whereas $\{\agent{\mathit{sched}}, \agent{\mathit{W1}}\}$ can, on their own, only ensure that $w$ holds in the third step on $\pi'$ (at the earliest). 
\end{example}

\paragraph{\HyperATLSS{} and \ATLS{}.}

\HyperATLSS{} subsumes \ATLS{}:

\begin{restatable}{proposition}{atlsInHyperatlss}\label{prop:to-hyperatlss}
	For every \ATLS{} formula $\varphi$, there exists an effectively computable \HyperATLSS{} formula $\varphi'$ such that for every CGS $\calG$, $\calG \models_{\text{\ATLS{}}} \varphi$ iff $\calG \models \varphi'$.
\end{restatable}

\section{Model Checking of \HyperATLSS{}}
 
While the extension of \ATLS{} to reason about hyperproperties required only minor modifications to its syntax, the subtle changes bring major complications in terms of model checking. 
In particular, the model-checking algorithm for \ATLS{} proposed by \citet{AlurHK02} is no longer applicable:
In \ATLS{},  checking if $\llangle A \rrangle \psi$ holds in some state $s$ can be reduced to the non-emptiness of the intersection of two tree automata.
One accepts all trees that represent possible strategies by the agents in $A$, and one accepts all trees whose paths satisfy the path formula $\psi$.
In \HyperATLSS{}, this is not possible: In a formula $\llangle A \rrangle_\xi \, \pi. \varphi$,  the satisfaction of $\varphi$ does not only depend on $\pi$ but also on path variables that are quantified before (outside).

\subsection{Alternating Automata}

Instead, our model-checking algorithm uses automata to  ``summarize'' path assignments that satisfy subformulas, similar to previous hyperlogics such as \HyperLTL{} \cite{FinkbeinerRS15,BeutnerFTacas23}, and \HyperATLS{} \cite{BeutnerF21,BeutnerF23}.
To handle the strategic interaction found in MASs, we rely on \emph{alternating automata}, i.e., automata that alternate between existential (non-deterministic) and universal transitions.

\begin{definition}
	An alternating parity automaton (APA) over alphabet $\Sigma$ is a tuple $\calA = (Q, q_0, \delta, c)$ where $Q$ is a finite set of states, $q_0 \in Q$ is an initial state, $c : Q \to \nat$ is a state coloring, and $\delta : Q \times \Sigma \to \bool^+(Q)$ is a transition function that maps pairs of state and letter to a positive boolean formula over $Q$ (denoted with $\bool^+(Q)$).
\end{definition}

Formally, we model the alternation in APAs by viewing the transitions as positive boolean formulas over states (i.e., formulas formed using only conjunctions and disjunctions). 
For example, if $\delta(q, l) = q_1 \lor (q_2 \land q_3)$, we can -- from state $q \in Q$ and upon reading letter $l \in \Sigma$ -- either move to state $q_1$ or move to \emph{both} $q_2$ and $q_3$ (i.e., spawn two copies of our automaton, one starting in state $q_2$ and one in $q_3$).
We write $\calL(\calA) \subseteq \Sigma^\omega$ for the set of all infinite words that are accepted by $\calA$, i.e., all infinite words where we can construct a run tree such that for all paths, the \emph{minimal} color that occurs infinity many times (as given by $c$) is \emph{even}. 
See \ifFull{\cite{Vardi95} and the appendix}{\cite{Vardi95,fullVersion}} for details.

\begin{figure}[!]
	\centering
	\begin{tikzpicture}
	
	\node[rectangle, draw, thick, label={[yshift=-1pt]above:{\small \textbf{0}}}] at (-0.7,0) (n0) {$q_0$};
	
	\node[circle, draw, inner sep=1pt] at (1,0) (d0) {\tiny$\land$};
	
	\coordinate[] (c0) at (1, -0.5);
	
	\coordinate[] (c1) at (-0.7, -0.5);
	
	\node[circle, draw, inner sep=1pt] at (2,0) (d1) {\tiny$\lor$};

	\node[rectangle, draw, thick, label={[yshift=1pt]below:{\small \textbf{1}}}] at (3,0.7) (n1) {$q_1$};
	
	\node[rectangle, draw, thick, label={[yshift=-1pt]above:{\small \textbf{1}}}] at (3,-0.7) (n2) {$q_2$};
	
	\node[rectangle, draw, thick,label={[yshift=-1pt]above:{\small \textbf{0}}}] at (4.5,0) (n3) {$q_3$};
	
	\draw[->, thick] (-1.2, 0) to (n0);

	\draw[->, thick] (n0) to node[above,align=center] {\small $a, b, c$} (d0);
	
	\draw[->, thick] (d0) to (d1);
	
	\draw[-, thick] (d0) to (c0);
	\draw[-, thick] (c0) to (c1);
	\draw[->, thick] (c1) to (n0);
	
	\draw[->, thick] (d1) to (n1);
	\draw[->, thick] (d1) to (n2);

	\draw[->, thick] (n1)  to node[above,align=center, sloped] {\small $a$} (n3);
	
	\draw[->, thick] (n2) to node[below,align=center, sloped] {\small $b$} (n3);

	\draw[->, thick, loop left] (n1) to node[left,align=center] {\small $b, c$} (n1);
	\draw[->, thick, loop left] (n2) to node[left,align=center] {\small $a, c$} (n2);

	\draw[->, thick, loop right] (n3) to node[right,align=center] {\small $a, b, c$} (n3);
	
\end{tikzpicture}

	\caption{Example APA over alphabet $\Sigma = \{a, b, c\}$.}\label{fig:apa}
\end{figure}

\begin{example}\label{ex:apa}
	Consider the APA in \Cref{fig:apa}.
	We display the color of each state and visualize transition formulas using $\land$ and $\lor$ nodes. 
	For example, $\delta(q_0, a) = \delta(q_0, b) = \delta(q_0, c) = q_0 \land (q_1 \lor q_2)$, i.e., whenever reading letter $a, b$, or $c$ in $q_0$ we start a fresh run from $q_0$ and \emph{at the same time} start a run from \emph{either} $q_1$ or $q_2$. 
	To derive the language of the APA, we first observe that state $q_1$ (resp.~$q_2$) accepts all words that contain at least one $a$ (resp.~$b$) (note that the color of $q_1, q_2$ is odd). 
	In the initial state $q_0$, we restart a run from $q_0$ \emph{and} transition to either $q_1$ or $q_2$.
	The language thus contains exactly those words that contain $a$ or $b$ infinitely often.
\end{example}

\paragraph{Deterministic Automata.}

Our model-checking algorithm relies on the fact that we can \emph{determinize} APAs. 
We say $\calA$ is a \emph{deterministic} parity automaton (DPA) if we can view $\delta$ as a function $Q \times \Sigma \to Q$ that assigns  a \emph{unique} successor state to each state, letter pair.

\begin{proposition}[\citet{MiyanoH84}]\label{prop:toDPA}
	For any APA $\calA$ with $n$ states, we can effectively compute a DPA $\calA'$ with at most $2^{2^{\calO(n)}}$ states such $\calL(\calA) = \calL(\calA')$.
\end{proposition}

\subsection{Model Checking Algorithm}

We are now in a position to outline our model-checking algorithm. 
A high-level description is given in \Cref{alg:main-alg}.
Here, we write $\lQ A \rQ$ as a shorthand for either $\llangle A \rrangle$ or $\llbracket A \rrbracket$.

\paragraph{Nested State Formulas.}
Initially, our algorithm recursively checks nested state formulas and replaces them with fresh atomic propositions \cite{EmersonH86}.
Concretely, given a closed state formula $\varphi = \lQ A_1 \rQ_{\xi_1} \pi_1 \ldots \lQ A_n \rQ_{\xi_n} \pi_n\ldot \psi$, we first extract all state formulas that are \emph{nested} in the path formula $\psi$ (\cref{line:extract-nested-formula}). 
For each nested state formula $\varphi'$, we \textbf{(1)} compute all states in which $\varphi'$ holds (using a recursive call to \lstinline[style=default, language=custom-lang]|modelCheck|); \textbf{(2)} mark all those states with a fresh atomic proposition $p_{\varphi'}$ by modifying the labeling function $L$ of $\calG$  (\cref{line:extend-label}); and \textbf{(3)} replace all occurrences of $\varphi'_\pi$ within $\psi$ with $(p_{\varphi'})_\pi$ (\cref{line:rename-label}).

\paragraph{Eliminating Path Quantification.}

Afterward, $\psi$ contains no nested state formulas, and we can tackle the strategic quantifiers. 
For each state $\dot{s} \in S$, we check if $\dot{s}, \emptyset \models_\calG \varphi$, and -- if it does -- add it to the solution set $\mathit{Sol}$ (\cref{line:emptinessCheck}).
Our main idea to check $\dot{s}, \emptyset \models_\calG \varphi$ is to iteratively eliminate paths $\pi_1, \ldots, \pi_n$ by \emph{simulating $\calG$ using the alternation available in APAs} while summarizing path assignments that satisfy the formula from the fixed state $\dot{s} \in S$.
To enable automata-based reasoning about path assignments, i.e., mappings $\Pi : V \to S^\omega$ for some $V \subseteq \pathVars$, we \emph{zip} such an assignment into an infinite word:
Given $\Pi : V \to S^\omega$ we define $\mathit{zip}(\Pi) \in (V \to S)^\omega$ as the infinite word over functions $V \to S$, defined by $\mathit{zip}(\Pi)(k)(\pi) := \Pi(\pi)(k)$ for $k \in \nat$.

\begin{definition}\label{def:equiv}
	Assume $\varphi$ is a \HyperATLSS{} formula with free path variables $V \subseteq \pathVars$.
	We say an automaton $\calA$ over $V \to S$ is \emph{$(\calG, \dot{s})$-equivalent to $\varphi$} if for every path assignment $\Pi : V \to S^\omega$ we have $\mathit{zip}(\Pi) \in \calL(\calA)$ if and only if $\dot{s}, \Pi \models_\calG \varphi$.
\end{definition}

Now assume that $\varphi = \lQ A_1 \rQ_{\xi_1} \pi_1 \ldots \lQ A_n \rQ_{\xi_n} \pi_n\ldot \psi$ is the state formula we want to check in state $\dot{s}$.
If we \emph{could} compute a $(\calG, \dot{s})$-equivalent automaton $\calA_{{\varphi}}$ for $\varphi$, we can immediately check whether $\dot{s}, \emptyset \models_\calG {\varphi}$ by testing if $\mathit{zip}(\emptyset) \in \calL(\calA_{{\varphi}})$.
Our main theoretical result is that we can construct such an automaton \emph{incrementally}:
We begin with a $(\calG, \dot{s})$-equivalent automaton $\calA_{\psi}$ for the body $\psi$; we then use $\calA_{\psi}$ to construct a $(\calG, \dot{s})$-equivalent automaton $\calA_{\lQ A_n \rQ_{\xi_n} \pi_n\ldot \psi}$ for $\lQ A_n \rQ_{\xi_n} \pi_n\ldot \psi$; and so forth, finally yielding the desired automaton $\calA_{\varphi}$ that is $(\calG, \dot{s})$-equivalent to $\varphi$. 
In each step, we apply the construction from the following theorem:

\begin{algorithm}[!t]
\caption{Model-checking algorithm for \HyperATLSS{}.}\label{alg:main-alg}

\begin{code}
def modelCheck($\calG$, $\varphi = \lQ A_1 \rQ_{\xi_1} \pi_1 \ldots \lQ A_n \rQ_{\xi_n} \pi_n\ldot \psi$):
@@@@for $\varphi'$ in nestedStateFormulas($\psi$) do (*\label{line:extract-nested-formula}*)
@@@@@@@@@$S_{\varphi'}$ = modelCheck($\calG$, $\varphi'$) (*\label{line:rec}*)
@@@@@@@@@$L$ = $\lambda s\ldot \begin{cases}
\begin{aligned}
&L(s) &\text{if} \quad &s \not\in S_{\varphi'}\\
&L(s) \cup \{p_{\varphi'}\} &\text{if} \quad &s \in S_{\varphi'}
\end{aligned}
\end{cases}$(*\label{line:extend-label}*)
@@@@@@@@@$\psi$ = $\psi\big[\varphi'_{\pi_1} / (p_{\varphi'})_{\pi_1} \big] \cdots\big[\varphi'_{\pi_n} / (p_{\varphi'})_{\pi_n} \big] $ (*\label{line:rename-label}*)
@@@@$\mathit{Sol}$ = $\emptyset$
@@@@for $\dot{s} \in S$ do
@@@@@@@@@$\calA$ = LTLtoAPA($\psi$) (*\label{line:ltlToApa}*)
@@@@@@@@@for $j$ from $n$ to $1$ do  
@@@@@@@@@@@@@@$\calA$ = product($\calG$, $\dot{s}$, $\calA$, $\lQ A_j\rQ_{\xi_j} \pi_j$) (*\label{line:product}*)
@@@@@@@@@if $\mathit{zip}(\emptyset) \in \calL(\calA)$ then (*\label{line:emptinessCheckIf}*)
@@@@@@@@@@@@@@$\mathit{Sol}$ = $\mathit{Sol} \cup \{\dot{s}\}$ (*\label{line:emptinessCheck}*)
@@@@return $\mathit{Sol}$
\end{code}
\end{algorithm}

\begin{figure*}[!t]
	
	\begin{minipage}[b]{0.5\linewidth}
		\begin{subfigure}{\linewidth}
			\centering
			\small
			\scalebox{0.95}{
				\begin{tikzpicture}
					\node[rectangle, draw, thick, label=below:{\small \textbf{1}}] at (0,0) (n0) {$q_0$};
					
					\node[rectangle, fill=mygreen!13, draw, thick, label=below:{\small \textbf{0}}] at (-3.5,0) (n1) {$q_1$};
					
					\node[rectangle, fill=myred!13, draw, thick,label=below:{\small \textbf{1}}] at (3.5,0) (n2) {$q_2$};
					
					\draw[->, thick] (-0.5, 0.5) to[] (n0);
					
					\draw[->, thick] (n0) to node[below,sloped,align=left] {\small $[\pi\! \mapsto \!s_2, \pi'\! \mapsto\! s_0]$,\\\small $[\pi\! \mapsto \!s_2, \pi'\! \mapsto\! s_1]$} (n1);
					
					\draw[->, thick] (n0) to node[below, sloped,align=left] {\small $[\pi \!\mapsto \!\_, \pi' \!\mapsto\! s_2]$} (n2);
					
					\draw[->, thick] (n0) to [loop above] node[above,align=left] {\small $[\pi \mapsto s_0, \pi' \mapsto s_0]$, $[\pi \mapsto s_0, \pi' \mapsto s_1]$,\\\small $[\pi \mapsto s_1, \pi' \mapsto s_0]$, $[\pi \mapsto s_1, \pi' \mapsto s_1]$} (n0);
					
					\draw[->, thick] (n1) to [loop above] node[above,align=center] {\small $\top$} (n1);
					\draw[->, thick] (n2) to [loop above] node[above,align=center] {\small $\top$} (n2);
					
				\end{tikzpicture}
			}
			\subcaption{}\label{fig:sub:apa-dpa}
		\end{subfigure}\\[4mm]
		\begin{subfigure}{\linewidth}
			\centering
			\small
			\scalebox{0.95}{
				\begin{tikzpicture}
					\draw[draw=mygreen, very thick, dashed, fill=mygreen!13] (-2.2,0.5) -- (-2.2, -0.5) -- (-3.3, -0.5) -- (-3.3, 0.5) -- (-2.2, 0.5);

					\draw[draw=mygreen, very thick, dashed, fill=mygreen!13] (-4.6,-1.9) -- (-3.4, -1.9) -- (-3.4, -3.3) -- (-4.6, -3.3) -- (-4.6,-1.9);
					\draw[draw=myred, very thick, dotted, fill=myred!13] (4.6,-1.9) -- (3.4, -1.9) -- (3.4, -3.3) -- (4.6, -3.3) -- (4.6,-1.9);
					
					\node[rectangle, draw, thick, label=above:{\small \textbf{1}}] at (0,0) (n0) {$q_0, s_0$};
					
					\node[rectangle, draw, thick, label=below:{\small \textbf{1}}] at (-1.3,-2.5) (n1) {$q_0, s_1$};

					\node[rectangle, draw, thick, label=below:{\small \textbf{1}}] at (1.3,-2.5) (n2) {$q_0, s_2$};
					
					\node[rectangle, draw, thick, label=above:{\small \textbf{0}}] at (-4,-2.5) (n3) {$q_1, s_2$};
					
					\node[rectangle, draw, thick, label=above:{\small \textbf{1}}] at (4,-2.5) (n5) {$q_2, s_0$};

					\draw[->, thick] (-0.5, 0.5) to (n0);
					
					\node[circle, draw, inner sep=1pt] at (0,-1.2) (d0) {\tiny$\land$};
					\node[circle, draw,inner sep=1pt] at (-1,-1.2) (d1) {\tiny$\lor$};
					\node[circle, draw,inner sep=1pt] at (0,-1.8) (d2) {\tiny$\lor$};
					
					\node[circle, draw,inner sep=1pt] at (1.8,-1.2) (d3) {\tiny$\lor$};
					
					\node[align=left,black!40] at (3.3,-0.2) (l1) {\scriptsize $[\agent{\mathit{sched}} \! \mapsto\! \texttt{ng}, \agent{\mathit{W1}}\! \mapsto\! \_]$};
					\draw[-,dotted, thick, black!40] (l1) -- (d3);
					
					\node[align=left,black!40] at (2.0,-1.6) (l2) {\scriptsize $[\agent{\mathit{sched}} \!\mapsto \!\texttt{g}, \agent{\mathit{W1}} \!\mapsto \!\texttt{r}]$};
					\draw[-,dotted, thick, black!40] (l2) -- (d2);
					
					\node[align=left,black!40] at (-3.0,-1.2) (l3) {\scriptsize $[\agent{\mathit{sched} \!}\mapsto\! \texttt{g}, \agent{\mathit{W1}}\! \mapsto\! \texttt{nr}]$};
					\draw[-,dotted, thick, black!40] (l3) -- (d1);
					
					\draw[->, thick] (n0) to node[right,align=left] {\small $[\pi \mapsto s_0]$,\\\small$[\pi \mapsto s_1]$} (d0);
					
					\draw[->, thick] (d0) to (d1);
					\draw[->, thick] (d0) to (d2);
					\draw[->, thick] (d0) to (d3);
					
					\draw[->, thick] (d1) -- (n0);
					\draw[->, thick] (d1) -- (n1);

					\draw[->, thick] (d2) to (n1);
					\draw[->, thick] (d2) to (n2);
					
					\draw[->, thick] (d3) |- (n0);
					
					\draw[->, thick] (n1) to node[below,align=center] {\small $[\pi \mapsto s_2]$} (n3);
					
					\draw[->, thick] (n1) to node[below,align=left] {\small $[\pi \mapsto s_0]$,\\\small$[\pi \mapsto s_1]$} (n2);

					\draw[->, thick] (n2) to node[below,align=left] {\small $[\pi \mapsto s_0]$,\\\small$[\pi \mapsto s_1]$,\\\small$[\pi \mapsto s_2]$} (n5);

					\node[circle, draw, inner sep=1pt] at (-2.5,0) (d4) {\tiny$\land$};
					\draw[->, thick] (n0) to node[above,align=center] {\small $[\pi \mapsto s_2]$} (d4);
					
					\draw[->, thick, dashed, black!60] (d4) to (-3, 0);
					\draw[->, thick, dashed, black!60] (d4) to (-3, -0.4);
					\draw[->, thick, dashed, black!60] (d4) to (-3, 0.4);

					\draw[->, thick, dashed, black!60] (n3) to (-4, -3.2);
					\draw[->, thick, dashed, black!60] (n3) to (-3.6, -3.2);
					\draw[->, thick, dashed, black!60] (n3) to (-4.4, -3.2);
					
					\draw[->, thick, dashed, black!60] (n5) to (4, -3.2);
					\draw[->, thick, dashed, black!60] (n5) to (3.6, -3.2);
					\draw[->, thick, dashed, black!60] (n5) to (4.4, -3.2);
				\end{tikzpicture}
			}
			\vspace{1mm}
			\subcaption{}\label{fig:sub:apa-apa}
		\end{subfigure}
	\end{minipage}%
	\begin{minipage}[b]{0.5\linewidth}
		\begin{subfigure}{\linewidth}
			\centering
			\small
			\scalebox{0.95}{
				\begin{tikzpicture}
					\node[rectangle, draw, thick, label=below:{\small \textbf{0}}] at (0,0) (n0) {$q_0'$};
					
					\node[rectangle, draw, thick, label=below:{\small \textbf{0}}] at (3,0) (n1) {$q_1'$};
					
					\node[rectangle, fill=myred!13, draw, thick, label=below:{\small \textbf{1}}] at (5.5,0) (n2) {$q_2'$};
					
					\node[rectangle, fill=mygreen!13, draw, thick, label=below:{\small \textbf{0}}] at (1.5,0.8) (n4) {$q_3'$};
					
					\draw[->, thick] (-0.65, 0) to (n0);
					
					\draw[->, thick] (n0) to node[below,sloped,align=left] {\small $[\pi \mapsto s_0]$,\\\small$[\pi \mapsto s_1]$} (n1);
					\draw[->, thick] (n1) to node[below,align=left] {\small $[\pi \mapsto s_0]$,\\\small$[\pi \mapsto s_1]$} (n2);
					
					\draw[->, thick] (n0) |- node[above,sloped,align=center,xshift=14pt] {\small $[\pi \!\mapsto \!s_2]$} (n4);
					\draw[->, thick] (n1) |- node[above,sloped,align=center,xshift=-14pt] {\small $[\pi \!\mapsto\! s_2]$} (n4);
					
					\draw[->, thick] (n4) to[loop above] node[above] {\small $\top$} (n4);
					
					\draw[->, thick] (n2) to[loop above] node[above] {\small $\top$} (n2);
					
				\end{tikzpicture}
			}
			\vspace{-1mm}
			\subcaption{}\label{fig:sub:apa-res-dpa}
		\end{subfigure}\\[3mm]
		\begin{subfigure}{\linewidth}
			\centering
			\small
			\scalebox{0.95}{
				\begin{tikzpicture}
					\draw[draw=mygreen, very thick, dashed, fill=mygreen!13] (1.4,-3.4) -- (2.9, -3.4) -- (2.9, -4.5) -- (1.4, -4.5) -- (1.4, -3.4);
					
					\draw[draw=myred, very thick, dotted, fill=myred!13] (-0.6,-3.4) -- (0.9, -3.4) -- (0.9, -4.5) -- (-0.6, -4.5) -- (-0.6, -3.4);
					
					\draw[draw=myred, very thick, dotted, fill=myred!13] (-1.4,-3.4) -- (-2.6, -3.4) -- (-2.6, -4.5) -- (-1.4, -4.5) -- (-1.4, -3.4);
					
					\node[rectangle, draw, thick, label=above:{\small \textbf{0}}] at (0,0) (n0) {$q'_0, s_0$};
					
					\node[rectangle, draw, thick, label=left:{\small \textbf{0}}] at (-2,-2.7) (n1) {$q'_1, s_0$};
					
					\node[rectangle, draw, thick, label=right:{\small \textbf{0}}] at (0,-2.7) (n2) {$q'_1, s_1$};
					
					\node[rectangle, draw, thick, label=right:{\small \textbf{0}}] at (2,-2.7) (n3) {$q'_1, s_2$};

					\node[rectangle, draw, thick, label=right:{\small \textbf{1}}] at (0,-3.7) (n4) {$q'_2, s_2$};
					
					\node[rectangle, draw, thick, label=right:{\small \textbf{0}}] at (2,-3.7) (n5) {$q'_3, s_0$};
					
					\draw[->, thick] (-0.85, 0) to (n0);
					
					\node[circle, draw, inner sep=1pt] at (0,-1.2) (d0) {\tiny$\lor$};
					
					\node[circle, draw,inner sep=1pt] at (-2,-2.0) (d1) {\tiny$\land$};
					\node[circle, draw,inner sep=1pt] at (0,-2.0) (d2) {\tiny$\land$};
					
					\node[circle, draw,inner sep=1pt] at (2,-2.0) (d3) {\tiny$\land$};
					
					\node[align=left,black!40] at (3,-1.5) (l1) {\scriptsize $[\agent{\mathit{sched}} \! \mapsto\! \texttt{g}, \agent{\mathit{W1}}\! \mapsto\! \texttt{r},\agent{\mathit{W2}}\! \mapsto\! \texttt{r}]$};
					\draw[-,dotted, thick, black!40] (l1) -- (d3);
					
					\node[align=left,black!40] at (2.5,-0.5) (l2) {\scriptsize $[\agent{\mathit{sched}} \!\mapsto \!\texttt{g}, \agent{\mathit{W1}} \!\mapsto \!\texttt{r},\agent{\mathit{W2}} \!\mapsto \!\texttt{nr}]$,\\\scriptsize $[\agent{\mathit{sched}} \!\mapsto \!\texttt{g}, \agent{\mathit{W1}} \!\mapsto \!\texttt{nr},\agent{\mathit{W2}} \!\mapsto \!\texttt{r}]$};
					\draw[-,dotted, thick, black!40] (l2) -- (d2);
					
					\node[align=left,black!40] at (-2.2,-0.7) (l3) {\scriptsize $[\agent{\mathit{sched}} \!\mapsto \!\texttt{ng}, \agent{\mathit{W1}} \!\mapsto \!\_,\agent{\mathit{W2}} \!\mapsto \!\_]$,\\\scriptsize $[\agent{\mathit{sched}} \!\mapsto \!\texttt{g}, \agent{\mathit{W1}} \!\mapsto \!\texttt{nr},\agent{\mathit{W2}} \!\mapsto \!\texttt{nr}]$};
					\draw[-,dotted, thick, black!40] (l3) -- (d1);
					
					\draw[->, thick] (n0) to node[right,align=left] {\small $\emptyset$} (d0);
					
					\draw[->, thick] (d0) to (d1);
					\draw[->, thick] (d0) to (d2);
					\draw[->, thick] (d0) to (d3);
					
					\draw[->, thick] (d1) -- (n1);
					\draw[->, thick] (d2) -- (n2);
					\draw[->, thick] (d3) -- (n3);
					
					\node[circle, draw, inner sep=1pt] at (-2,-3.7) (d4) {\tiny$\lor$};

					\draw[->, thick] (n2) to node[right,align=left] {\small $\emptyset$} (n4);
					
					\draw[->, thick] (n3) to node[right,align=left] {\small $\emptyset$} (n5);
					
					\draw[->, thick] (n1) to node[right,align=left,yshift=0.6mm] {\small $\emptyset$} (d4);
					
					\draw[->, thick, dashed, black!60] (n4) to (0,-4.4);
					\draw[->, thick, dashed, black!60] (n4) to (-0.4,-4.4);
					\draw[->, thick, dashed, black!60] (n4) to (0.4,-4.4);
					
					\draw[->, thick, dashed, black!60] (n5) to (2,-4.4);
					\draw[->, thick, dashed, black!60] (n5) to (1.6,-4.4);
					\draw[->, thick, dashed, black!60] (n5) to (2.4,-4.4);
					
					\draw[->, thick, dashed, black!60] (d4) to (-2,-4.4);
					\draw[->, thick, dashed, black!60] (d4) to (-1.6,-4.4);
					\draw[->, thick, dashed, black!60] (d4) to (-2.4,-4.4);		
				\end{tikzpicture}
			}
			\subcaption{}\label{fig:running-cont-final}
		\end{subfigure}%
	\end{minipage}
	
	\caption{Illustration of our model-checking algorithm on \Cref{ex:cgs}.
		In \Cref{fig:sub:apa-dpa}, we depict a DPA over alphabet $\{\pi, \pi'\} \to \{s_0, s_1, s_2\}$ for the body $(\neg \mathit{w}_{\pi'}) \ltlU (\neg \mathit{w}_{\pi'} \land \mathit{w}_{\pi} )$.
		In \Cref{fig:sub:apa-apa}, we sketch the APA over alphabet $\{\pi\} \to \{s_0, s_1, s_2\}$ constructed using \Cref{theo:construction} that is $(\calG, s_0)$-equivalent to subformula $\llbracket \agent{\mathit{sched}}, \agent{\mathit{W1}} \rrbracket \, \pi'\ldot (\neg \mathit{w}_{\pi'}) \ltlU (\neg \mathit{w}_{\pi'} \land \mathit{w}_{\pi})$. 
		In \Cref{fig:sub:apa-res-dpa}, we depict a DPA that is equivalent to the APA in \Cref{fig:sub:apa-apa}.
		Lastly, in \Cref{fig:running-cont-final}, we sketch the APA (over singleton alphabet $\emptyset \to \{s_0, s_1, s_2\}$) constructed using \Cref{theo:construction} that is $(\calG, s_0)$-equivalent to $\llangle \agent{\mathit{sched}}, \agent{\mathit{W1}}, \agent{\mathit{W2}} \rrangle \, \pi \ldot\llbracket \agent{\mathit{sched}}, \agent{\mathit{W1}} \rrbracket \, \pi'\ldot (\neg \mathit{w}_{\pi'}) \ltlU (\neg \mathit{w}_{\pi'} \land \mathit{w}_{\pi})$. 
}
\end{figure*}

\begin{restatable}{theorem}{main}\label{theo:construction}
	Assume that $\varphi = \lQ A \rQ_\xi \, \pi\ldot \varphi'$ and let $\calA_{\varphi'}$ be an APA over alphabet $(V \cup \{\pi\} \to S)$ that is $(\calG, \dot{s})$-equivalent to $\varphi'$. 
	We can effectively construct an APA $\calA_{\varphi}$ over alphabet $V \to S$ that is $(\calG, \dot{s})$-equivalent to $\varphi$. 
	The size of $\calA_{\varphi}$ is at most double exponential in the size of $\calA_{\varphi'}$.
\end{restatable}
\begin{proof}
	Let $\calA^\mathit{det}_{\varphi'} = (Q, q_0, \delta, c)$ be a DPA equivalent to $\calA_{\varphi'}$ (cf.~\Cref{prop:toDPA}).
	We define $\calA_\varphi = (Q \times S, (q_0, \dot{s}), \delta', c')$ where $c'(q, s) := c(q)$ and $\delta'$ is defined as follows: 
	If $\varphi = \llangle A \rrangle_\xi \, \pi\ldot \varphi'$ we define $\delta'\big((q, s), l\big)$ for $l : V \to S$ as
	\begin{align*}
		 &\bigvee_{\substack{\bA : A \to \moves \\\forall \agent{i}, \agent{j} \in A\ldot (\agent{i}, \agent{j}) \in \xi  \\\Rightarrow\bA(\agent{i}) = \bA(\agent{j})}} \;\; \bigwedge_{\substack{\bA' : \overline{A} \to \moves \\\forall \agent{i}, \agent{j} \in \overline{A}\ldot (\agent{i}, \agent{j}) \in \xi  \\\Rightarrow\bA'(\agent{i}) = \bA'(\agent{j})}} \!\!\! \Big(\delta\big(q, l[\pi \mapsto s]\big), \kappa\big(s, \bA \oplus \bA'\big)\Big)
	\end{align*}
	Conversely, if $\varphi = \llbracket A \rrbracket_\xi \, \pi\ldot \varphi'$ we define $\delta'\big((q, s), l\big)$ as
	\begin{align*}
		&\bigwedge_{\substack{\bA : A \to \moves \\\forall \agent{i}, \agent{j} \in A\ldot (\agent{i}, \agent{j}) \in \xi  \\\Rightarrow\bA(\agent{i}) = \bA(\agent{j})}} \;\; \bigvee_{\substack{\bA' : \overline{A} \to \moves \\\forall \agent{i}, \agent{j} \in \overline{A}\ldot (\agent{i}, \agent{j}) \in \xi  \\\Rightarrow\bA'(\agent{i}) = \bA'(\agent{j})}} \!\!\! \Big(\delta\big(q, l[\pi \mapsto s]\big), \kappa\big(s, \bA \oplus \bA'\big)\Big)
	\end{align*}
	The intuition behind our construction is that we can simulate the strategic quantification at the level of states, similar to what is possible in \HyperATLS{} \cite{BeutnerF21,BeutnerF23}.
	Let us take $\varphi = \llangle A \rrangle_\xi \, \pi\ldot \varphi'$ as an example. The desired automaton $\calA_{\varphi}$ should accept a word $u \in (V \to S)^\omega$ iff there exists a strategy vector $\bF : A \to \strats{\calG}$ that respects $\xi$ and for all paths $\pi$ compatible with $\bF$, the extended zipped path assignment (a word in $(V \cup \{\pi\} \to S)^\omega$) is accepted by $\calA_{\varphi'}^\mathit{det}$.
	In our constructions, we track the current state $q$ of $\calA_{\varphi'}^\mathit{det}$ and simulate $\calG$ by keeping track of the current state $s$. 
	When in state $(q, s)$, we update the automaton state according to the transition function of $\calA^\mathit{det}_{\varphi'}$ using the current state $s$ for path variable $\pi$. 
	To update the state of $\calG$, we simulate the strategic behavior: \textbf{(1)} we disjunctive fix actions for each agent in $A$ via a function $\bA$ and ensure that all sharing constraints hold; \textbf{(2)} we conjunctively choose actions for $\overline{A}$ as a function $\bA'$ (subject to the sharing constraints); and \textbf{(3)} we update the system state to $\kappa(s, \bA \oplus \bA')$.
	
	Arguing that $\calA_{\varphi}$ is $(\calG, \dot{s})$-equivalent to $\varphi$ is based on the determinacy of the underlying game. 
	As we work in a setting of complete information (where all agents observe the state of the CGS), we can replace the existential quantification over strategies for $A$ (as in the semantics of $\llangle A \rrangle_\xi \, \pi\ldot \varphi'$) with existential quantification over actions for $A$ in each step (as used in the disjunctive choice in the definition of $\calA_{\varphi}$).
	We give a formal proof in \ifFull{\Cref{app:proof-main}}{\cite{fullVersion}}.
	
	The size of $\calA_{{\varphi}}$ is linear in the size of $\calG$ and $\calA_{{\varphi'}}^\mathit{det}$ (which itself is doubly exponential in $\calA_{{\varphi'}}$, cf.~\Cref{prop:toDPA}).
\end{proof}

For a formula $\varphi = \lQ A \rQ_\xi \, \pi \ldot \varphi'$ and APA $\calA$ that is $(\calG, \dot{s})$-equivalent to $\varphi'$, let \lstinline[style=default, language=custom-lang]|product($\calG$, $\dot{s}$, $\calA$, $\lQ A \rQ_\xi \, \pi$)| be the APA that is $(\calG, \dot{s})$-equivalent to $\varphi$ constructed using \Cref{theo:construction}. 
In \Cref{alg:main-alg}, we start with an APA that is equivalent to $\psi$ (line \ref{line:ltlToApa}), and apply \lstinline[style=default, language=custom-lang]|product| to iteratively compute $(\calG, \dot{s})$-equivalent automata for subformulas $\lQ A_j \rQ_{\xi_j} \pi_j \ldots \lQ A_n \rQ_{\xi_n} \pi_n\ldot \psi$ for $j$ from $n$ to $1$  (\cref{line:product}).
After the loop, we are left with an APA $\calA$ over singleton alphabet $(\emptyset \to S)$ that is $(\calG, \dot{s})$-equivalent to $\varphi$; we can thus decide if $\dot{s}, \emptyset \models_\calG \varphi$ by checking if $\mathit{zip}(\emptyset) \in \calL(\calA)$ (\cref{line:emptinessCheckIf}).

\begin{restatable}{proposition}{correctProp}\label{prop:correctProp}
	For every CGS $\calG = (S, s_0, \moves, \kappa, L)$ and closed \HyperATLSS{} formula $\varphi$, we have 
	\begin{align*}
		\text{\lstinline[style=default, language=custom-lang]|modelCheck($\calG$,$\,$$\varphi$)|} = \big\{s \in S \mid s, \emptyset \models_\calG \varphi \big\}.
	\end{align*}
\end{restatable}

\paragraph{Complexity.}

Each application of \lstinline[style=default, language=custom-lang]|product| increases the size of $\calA$ by (in the worst case) two exponents.
Checking a \HyperATLSS{} formula with $n$ nested quantifiers is thus in $2n$-\EXPTIME{}, and MC for general formulas is non-elementary.
As \HyperATLSS{} subsumes \HyperATLS{}, we get a matching non-elementary \emph{hardness} \cite{BeutnerF23}.
\HyperATLSS{} is thus more expressive (and also much harder to model-check) than \ATLS{}.
We stress that the non-elementary complexity of \HyperATLSS{} stems from its ability to quantify over \emph{arbitrarily} many paths. 
In most properties of interest, we quantify over few paths (cf.~\Cref{sec:eval}), which results in a much lower (elementary) complexity.
In particular, if we apply \Cref{alg:main-alg} to an \ATLS{}-equivalent formula (cf.~\Cref{prop:to-hyperatlss}), we deal with a \emph{single} nested quantifier ($n = 1$) and thus match the $2$-\EXPTIME{} MC complexity known for \ATLS{} \cite{AlurHK02}.

\subsection{Model Checking the Running Example}\label{sec:sub:running}

We illustrate our MC construction on the formula from \Cref{ex:cgs}. 
In a first step, we translate the body $(\neg \mathit{w}_{\pi'}) \ltlU (\neg \mathit{w}_{\pi'} \land \mathit{w}_{\pi})$ to a DPA over alphabet $(\{\pi, \pi'\} \! \to \!\{s_0, s_1, s_2\})$, depicted in \Cref{fig:sub:apa-dpa}. 
Afterward, we can follow the construction from \Cref{theo:construction} to obtain an APA over $(\{\pi\} \to \{s_0, s_1, s_2\})$ that is $(\calG, s_0)$-equivalent to subformula $\llbracket \agent{\mathit{sched}}, \agent{\mathit{W1}}\rrbracket \pi'\ldot (\neg \mathit{w}_{\pi'}) \ltlU (\neg \mathit{w}_{\pi'} \land \mathit{w}_{\pi})$. 
We depict a sketch in \Cref{fig:sub:apa-apa}. 
We start in state $(q_0, s_0)$. 
When reading letter $[\pi \!\mapsto\! s_2]$, we update the automaton state to $q_1$, so -- as $q_1$ is an accepting sink -- every run from such states is accepting.
To aid readability, we stop exploration as soon as the automaton state equals $q_1$ or $q_2$ and mark them with a green (dashed border) and red (dotted border) box to represent acceptance and rejection, respectively. 
When reading letters $[\pi \!\mapsto \!s_0]$ or $[\pi \!\mapsto \!s_1]$, we remain in automaton state $q_0$.
However, to update the state of the CGS, we need to simulate the strategic behavior within the CGS. 
As we quantify \emph{universally} over strategies for $\{\agent{\mathit{sched}}, \agent{\mathit{W1}}\}$, we conjunctively consider all possible action vectors $\{\agent{\mathit{sched}}, \agent{\mathit{W1}}\} \to \moves$. 
For each such action vector, we can then disjunctively choose an action for $\agent{\mathit{W2}}$. 
In our visualization in	\Cref{fig:sub:apa-apa}, we use decision nodes (as in \Cref{ex:apa}); 
for the reader's convenience, we label each conjunctive choice with the corresponding partial action vector.
For example, if we conjunctively pick the (partial) action vector $[\agent{\mathit{sched}} \!\mapsto \!\texttt{g}, \agent{\mathit{W1}} \!\mapsto \!\texttt{r}]$, agent $\agent{\mathit{W2}}$ can (disjunctively) move to either $(q_0, s_1)$ or $(q_0, s_2)$.

To better understand of the APA we have just constructed, we can translate it to some equivalent DPA, depicted in \Cref{fig:sub:apa-res-dpa}.
For this DPA, we can see that a path assignment is accepted iff $s_2$ (i.e., the unique state where AP $w$ holds) occurs within the first two steps on $\pi$.
This exactly matches the intuition of $(\calG, s_0)$-equivalence:
A path assignment $\Pi : \{\pi\} \to \{s_0, s_1, s_2\}^\omega$ satisfies $s_0, \Pi \models_\calG \llbracket \agent{\mathit{sched}}, \agent{\mathit{W1}} \rrbracket \, \pi'\ldot (\neg \mathit{w}_{\pi'}) \ltlU (\neg \mathit{w}_{\pi'} \land \mathit{w}_{\pi})$	iff $s_2$ occurs within the first two steps on $\pi$.  
If $s_2$ does \emph{not} hold on in the first two steps, $\{\agent{\mathit{sched}}, \agent{\mathit{W1}}\}$ can ensure that $s_2$ holds in the third step on $\pi'$ and thus violate the property. 
	
We can use the DPA in \Cref{fig:sub:apa-res-dpa} and, again, apply \Cref{theo:construction} to the outermost quantifier $\llangle\agent{\mathit{sched}}, \agent{\mathit{W1}}, \agent{\mathit{W2}} \rrangle \, \pi$, resulting in the APA over singleton alphabet $(\emptyset \! \to \! \{s_0, s_1, s_2\})$ sketched in \Cref{fig:running-cont-final}. 
Here, we \emph{disjunctively} pick an action vector $\{\agent{\mathit{sched}}, \agent{\mathit{W1}}, \agent{\mathit{W2}}\} \to \moves$ (annotated at each decision node).
As there are no agents in $\overline{\{\agent{\mathit{sched}}, \agent{\mathit{W1}}, \agent{\mathit{W2}}\}}$, each action vector yields a unique successor state.
It is easy to see that this APA accepts $\mathit{zip}(\emptyset) = \emptyset^\omega$, proving $\calG \models \varphi$.

\section{Implementation and Experiments}\label{sec:eval}

We have implemented our algorithm in a tool we call \tool{}.
As input, our tool reads MASs in the form of ISPL models \cite{LomuscioQR09}.
For automata operations (in particular, the translation from alternating to deterministic automata), we use \texttt{spot}; an actively-maintained automata library \cite{Duret-LutzRCRAS22}. 
To check APAs over the \emph{singleton} alphabet ($\emptyset \to S$) for emptiness, we use the parity-game solver \texttt{oink} \cite{Dijk18}.
All results were obtained on a 3.60GHz Xeon\textregistered{} CPU (E3-1271) with 32GB of memory running Ubuntu 20.04.
  
\subsection{Model Checking \ATLS{}}\label{sec:sub:atls}

\begin{table}
	\centering
	\small
\begin{tabular}{lllll}
	\toprule
	$\boldsymbol{n}$ & $\boldsymbol{|S|}$ & $\boldsymbol{|S_\mathit{reach}|}$ &  $\boldsymbol{t}_{\mcmassl{}}$& $\boldsymbol{t}_\tool{}$  \\
	\midrule
	2 &72 & 9 & \textbf{0.11} & 0.41  \\
	3 &432& 21 & 6.64 & \textbf{2.06}  \\
	4 &2592& 49 & 322.7 & \textbf{24.3}   \\
	5 &15552& 113& TO & \textbf{347.1} \\
	\bottomrule
\end{tabular}

\caption{We compare \tool{} and \mcmassl{}. We give the size of the system ($|S|$), the size of the reachable fragment ($|S_\mathit{reach}|$), and the verification times in seconds. The timeout (TO) is set to 1h.
}\label{tab:results-sl}
\end{table}

In our first experiment, we want to compare the performance of \tool{} against existing tools for strategic properties.
This requires us to consider \emph{non-hyper} properties in the form of \ATLS{} specification (as no existing tool can handle hyperproperties).
Concretely, we compare with  \mcmassl{}, a solver for a fragment of strategy logic \cite{CermakLM15}.
We use the same benchmark family used by \citet{CermakLM15}, describing a parametric scheduling problem consisting of agents $\agents = \{ \agent{\mathit{sched}}, \agent{y_1}, \ldots, \agent{y_n}\}$ for $n \in \nat$.
We check the following \HyperATLSS{} formula
\begin{align*}
	\llangle \agent{\mathit{sched}} \rrangle \pi \ldot \bigwedge_{i=1}^n \ltlG \Big( \langle \mathit{wt}, i\rangle_\pi \to \ltlF \neg \langle \mathit{wt}, i\rangle_\pi \Big)
\end{align*}
The formula states that whenever agent $\agent{y_i}$ waits (modeled by AP $\langle \mathit{wt}, i\rangle$), it will eventually not wait anymore, i.e., the scheduler has a strategy that avoids starvation of all agents.
This formula is equivalent to the strategy logic specification used by \citet{CermakLM15}.

We check the \HyperATLSS{} formulas (and the equivalent strategy logic specifications) with \tool{} and \mcmassl{} for varying values of $n$.
We use the ``optimized'' algorithm in \mcmassl{} that decomposes the formula as much as possible \cite[\S4]{CermakLM15}.
The results are given in \Cref{tab:results-sl}. 
We observe that \tool{} performs much faster than \mcmassl{}, which we largely accredit to the very efficient backend solvers in \texttt{spot} and \texttt{oink}.

\subsection{Model Checking Hyperproperties}\label{sec:sub:eval-hyper}

\begin{table*}[!t]
	\centering
	\small
	\begin{tabular}{l@{\hspace{5mm}}c@{\hspace{3mm}}c@{\hspace{12mm}}c@{\hspace{3mm}}c@{\hspace{12mm}}c@{\hspace{3mm}}c@{\hspace{12mm}}c@{\hspace{3mm}}c@{\hspace{12mm}}c@{\hspace{4mm}}c}
		\toprule
		& \multicolumn{2}{@{\hspace{-10mm}}c}{\textbf{Optimality I}} & \multicolumn{2}{@{\hspace{-10mm}}c}{\textbf{Optimality II}} & \multicolumn{2}{@{\hspace{-10mm}}c}{\textbf{Optimality III}} & \multicolumn{2}{@{\hspace{-10mm}}c}{\textbf{OD}} & \multicolumn{2}{@{\hspace{-2mm}}c}{\textbf{GE}}\\
		\cmidrule[1pt](l{-2mm}r{10mm}){2-3}
		\cmidrule[1pt](l{-2mm}r{10mm}){4-5}
		\cmidrule[1pt](l{-2mm}r{10mm}){6-7}
		\cmidrule[1pt](l{-2mm}r{10mm}){8-9}
		\cmidrule[1pt](l{-2mm}){10-11}
		\textbf{ISPL Model}  & $\boldsymbol{t}_\mathit{avg}$ & $\boldsymbol{t}_\mathit{max}$ & $\boldsymbol{t}_\mathit{avg}$ & $\boldsymbol{t}_\mathit{max}$ & $\boldsymbol{t}_\mathit{avg}$ & $\boldsymbol{t}_\mathit{max}$ & $\boldsymbol{t}_\mathit{avg}$ & $\boldsymbol{t}_\mathit{max}$ & $\boldsymbol{t}_\mathit{avg}$ & $\boldsymbol{t}_\mathit{max}$ \\
		\midrule
		\textsc{bit-transmission}   & $0.38$ & $0.41$  & 0.39 & 0.42  & 0.39 & 0.44  & 0.39 & 0.44  & 0.38 & 0.42 \\
		\textsc{book-store}    & $0.39$ & $0.42$  & 0.40 & 0.47  & 0.40 & 0.44  &0.39 & 0.42 & 0.39 & 0.43\\
		\textsc{card-game}   & $0.38$ & $0.39$  & 0.38 & 0.39  & 0.41 & 0.48  & 0.39 & 0.46 & 0.36 & 0.37\\
		\textsc{dining cryptographers}   & 0.70 & 0.77  & 0.68 &0.85   & 0.70 &0.77  &0.69&0.74  & 0.69 &1.10\\
		\textsc{muddy-children}   & 0.36 & 0.44  & 0.36 & 0.40  & 0.36 & 0.42  & 0.36 & 0.42  & 0.36 & 0.42\\
		\textsc{simple-card-game}   & 0.35 &0.38  & 0.35 &0.37  & 0.35 & 0.38  &0.35&0.38 & 0.34 &0.35\\
		\textsc{software-development}  & - & - & - & - & - & - & -  & - & -   & -\\
		\textsc{strongly-connected}   & 0.35 & 0.41  & 0.34 &0.37  & 0.35 &0.37   & 0.37 & 0.40  & 0.34 & 0.38\\
		\textsc{tianji-horse-racing-game}  & 0.38 &0.45  & 0.37 & 0.39   & 0.37 & 0.40   &0.37 & 0.40 & 0.37 & 0.42\\
		\midrule
		\textsc{scheduler-2}   & 0.47 &0.51  & 0.46 &0.48  & 0.95 &1.35  & 0.47 &0.53  & 0.48 &0.51\\
		\textsc{scheduler-3}   & 2.33 & 2.85  & 2.29 & 2.70  & 9.72 & 20.1  & 2.12& 2.64  & 2.01 & 2.15\\
		\textsc{scheduler-4}   & 29.5 &32.7  & 24.5 & 24.7  & 31.2 & 35.2  & 28.36 & 58.7  & 24.6 & 25.1\\
		\bottomrule
	\end{tabular}

\caption{For each ISPL model \cite{LomuscioQR09}, we display the average time ($t_\mathit{avg}$) and the maximal time ($t_\mathit{max}$) time (in seconds) needed by \tool{} across the 20 random instances sampled from each template.}\label{tab:results-hyper}
\end{table*}

In this section, we challenge \tool{} with interesting hyperproperties.
As underlying MAS models, we use the ISPL models used by \mcmas{} \cite{CermakLM15,LomuscioQR09} and design a range of specification templates that model interesting use cases of \HyperATLSS{}.
We emphasize that not every template models realistic properties in each of the ISPL instances. 
However, our evaluation \textbf{(1)} demonstrates that \HyperATLSS{} can express interesting properties, and \textbf{(2)} empirically shows that \tool{} can check such properties in existing ISPL models (confirming this via further real-world scenarios is interesting future work).
Note that none of the properties falls in the self-composition fragment of \HyperATLS{} \cite{BeutnerF21,BeutnerF23}, which formed the largest fragment supported by previous tools. 

\paragraph{Optimality I.}

As argued in \Cref{sec:intro} and \Cref{ex:cgs}, a particular strength of \HyperATLSS{} is the ability to compare the power of different coalitions. 
For $A, A' \subseteq \agents$ and $\mathit{tgt} \in \ap$ (modeling the target), we check
\begin{align*}
	\llangle A \rrangle \, \pi \ldot \llbracket A' \rrbracket \, \pi'\ldot (\neg \mathit{tgt} _{\pi'}) \ltlU (\neg \mathit{tgt} _{\pi'} \land \mathit{tgt} _{\pi}).
\end{align*} 

\paragraph{Optimality II.}

Using the strategy sharing in \HyperATLSS{}, we can also check if a coalition can achieve the goal equally fast despite using the same strategy for all agents (cf.~\Cref{ex:sharing}). 
For a group of agents $A$ and $\mathit{tgt} \in \ap$, we use \tool{} to check 
\begin{align*}
	\llangle A \rrangle_{\{(\agent{i}, \agent{j}) \mid \agent{i}, \agent{j} \in A\}} \, \pi \ldot \llbracket A \rrbracket \, \pi'\ldot  (\neg \mathit{tgt} _{\pi'}) \ltlU  \mathit{tgt} _{\pi}.
\end{align*} 

\paragraph{Optimality III.}

Likewise, we can express that coalition $A$ can reach a target state at strictly more time points than coalition $A'$ as
\begin{align*}
	\llangle A \rrangle \, \pi \ldot \llbracket A' \rrbracket \, \pi'\ldot &\ltlG (\mathit{tgt} _{\pi'} \to \mathit{tgt} _{\pi}) \land \ltlF (\neg \mathit{tgt} _{\pi'} \land \mathit{tgt} _{\pi}).
\end{align*}

\paragraph{Observational Determinism (OD).}

An important property in the context of security in MASs is \emph{observational determinism} \cite{ZdancewicM03}.
For example, assume we have a system that contains a controller agent $\agent{\mathit{cnt}}$ and an AP $h$ that models a high-security value of the system. 
We want to ensure that the value of $h$ is in control of $\agent{\mathit{cnt}}$, which we can express in \HyperATLSS{} as 
\begin{align*}
	\llangle \agent{\mathit{cnt}} \rrangle \, \pi \ldot \llangle \agent{\mathit{cnt}} \rrangle \, \pi'\ldot \ltlG (h_\pi \leftrightarrow h_{\pi'}).
\end{align*} 
That is, $\agent{\mathit{cnt}}$ has a strategy to construct $\pi$ such that in a second execution, $\agent{\mathit{cnt}}$ can ensure the same sequence of values for $h$ (despite the other agents potentially acting differently).

\paragraph{Good-Enough Synthesis (GE).}

In many scenarios, there does not exist a strategy that wins in all situations. 
Instead, we often look for strategies that are \emph{good-enough} (GE), i.e., strategies that win on every possible input sequence for which a winning output sequence exists \cite{AlmagorK20,AminofGR21,LiTVZ21}. 
We can express the existence of a GE strategy for $A \subseteq \agents$ in \HyperATLSS{} as  
\begin{align*}
	\llangle A \rrangle \, \pi \ldot \llangle \emptyset \rrangle \, \pi'\ldot (\ltlG (\mathit{in}_{\pi} \leftrightarrow \mathit{in}_{\pi'}) \land \ltlF \mathit{tgt}_{\pi'}) \to \ltlF \mathit{tgt}_\pi.
\end{align*} 
That is, $A$ has a strategy for $\pi$ such that if any other (universally quantified) path $\pi'$ agrees on the input $\mathit{in} \in \ap$ with $\pi$ and wins (e.g., reaches a state where $\mathit{tgt} \in \ap$ holds), then $\pi$ must win as well.
Phrased differently, $\pi$ only needs to win, provided some path with the same inputs \emph{can} win.

\paragraph{Results.}

For each ISPL model, we sample 20 random \HyperATLSS{} formulas from each of the templates and display the verification times in \Cref{tab:results-hyper}.
We observe that \tool{} can verify almost all instances within a few seconds.
Even on the challenging scheduler instances, verification of complex hyperproperties is only slightly more expensive than checking non-hyper \ATLS{} formulas (cf.~\Cref{tab:results-sl}).
The only exception is the \textsc{software-development} model; this model consists of roughly 15k states, which is too large for any automata-based representation. 
We stress that already in the non-hyper realm, \mcmassl{} cannot verify (even simple) \ATLS{} and strategy logic specifications in the \textsc{software-development} model and is only applicable to \ATL{} and \CTL{} properties.

\section{Conclusion}

Starting with the seminal work on \ATLS{}, the past decade has seen immense progress in (temporal) logic-based frameworks that provide rigorous and formal guarantees in MASs.
Thus far, most logics focus on a purely path-based view where we reason about the strategic (in)ability of agents.
However, many important properties require reasoning about multiple paths at the same time and investigating scenarios where agents share strategies. 
We have presented \HyperATLSS{}, a powerful logic that bridges this gap. 
Our logic: \textbf{(1)} can express many important properties such as optimality requirements, OD, and GE; \textbf{(2)} admits decidable model checking; and \textbf{(3)} can be checked fully automatically using \tool{}.

For the future, it is interesting to cast even more properties in a \emph{unified} framework using (hyper)logics such as \HyperATLSS{} (similar to what has been done for \ATL{}/\ATLS{}) and explore even more scalable verification approaches for \HyperATLSS{} using, e.g., symbolic techniques.

\section*{Acknowledgments}

This work was supported by the European Research Council (ERC) Grant HYPER (101055412), and by the German Research Foundation (DFG) as part of TRR 248 (389792660).

\bibliography{references}

\iffullversion
\newpage

\appendix

\section{\HyperATLSS{} and \ATLS{}}

\atlsInHyperatlss*
\begin{proof}
	Let $\dot{\pi} \in \pathVars$ be a fixed path variable.
	We recursively translate \ATLS{} path and state formulas as follows:
	Given a \ATLS{} path formula $\psi$ we define the
	\begin{align*}
		\atlsToHyper{a} &:= a_{\dot{\pi}} &  \atlsToHyper{\neg \psi} &:= \neg \atlsToHyper{\psi} & \atlsToHyper{\psi_1 \land \psi_2} &:= \atlsToHyper{\psi_1} \land \atlsToHyper{\psi_2}\\
		\atlsToHyper{\varphi} &:=  \atlsToHyper{\varphi}_{\dot{\pi}}  & \atlsToHyper{\ltlN \psi} &:= \ltlN \atlsToHyper{\psi} &  \atlsToHyper{\psi_1 \ltlU \psi_2} &:= \atlsToHyper{\psi_1} \ltlU \atlsToHyper{\psi_2}
	\end{align*}
	Likewise, we translate \ATLS{} state formulas as follows:
	\begin{align*}
		\atlsToHyper{\llangle A \rrangle \, \psi} &:= \llangle A \rrangle_\emptyset \, \dot{\pi}\ldot \atlsToHyper{\psi}\\
		\atlsToHyper{\llbracket A \rrbracket \, \psi} &:= \llbracket A \rrbracket_\emptyset \, \dot{\pi}\ldot \atlsToHyper{\psi}
	\end{align*}
	That is, whenever \ATLS{} implicitly quantifies over a path, we use the path variable $\dot{\pi}$. 
	Each quantification $\llangle A \rrangle$ and $\llbracket A \rrbracket$  then constructs this path $\dot{\pi}$ under \emph{no} sharing constraints. 
	It is easy to see that $\calG \models_{\text{\ATLS{}}} \varphi$ iff $\calG \models \atlsToHyper{\varphi}$ using a simple induction.
\end{proof}

\section{Alternating Automata}\label{app:apa}

In this section, we formalize the semantics of APAs.

\begin{definition}[Positive Boolean Formula]
	For a set $Q$, we write $\bool^+(Q)$ for the set of all positive boolean formulas over $Q$, i.e., all formulas generated by the following grammar
	\begin{align*}
		\theta := q \mid \theta_1 \land \theta_2 \mid \theta_1 \lor \theta_2
	\end{align*}
	where $q \in Q$.
	Given a subset $X \subseteq Q$ and $\theta \in \bool^+(Q)$, we write $X \models \theta$ if the assignment that maps all states in $X$ to $\top$ and those in $Q \setminus X$ to $\bot$ satisfies $\Psi$. 
	For example $\{q_0, q_1\} \models q_0 \land (q_1 \lor q_2)$. 
\end{definition}

\paragraph{Trees.}

Given the alternating nature of an APA, a run on a word is not an infinite sequence of states (as is usual in non-deterministic automata) but an infinite \emph{tree}. 
Intuitively, each branch in the tree will correspond to a creation of multiple automata runs; as needed for universal branching in the automaton.

We formalize a tree as a subset $T \subseteq \nat^*$ with root $\epsilon \in T$ (where $\epsilon$ denotes the empty sequence).
We refer to elements $\tau \in T$ as nodes and let $|\tau| \in \nat$ be the depth of node $\tau$ (i.e., the length of the sequence).
We define $\mathit{children}_T(\tau) := \{\tau \cdot n \in T \mid n \in \nat\}$ as the set of immediate successors of $\tau$ in $T$.
A $Q$-labeled tree is pair $(T, \ell)$ where $T$ is a tree and $\ell : T \to Q$ is a labeling of nodes with $Q$.  

\paragraph{Run Trees.}

We can now formally define when a $Q$-labeled tree denotes a run of an APA. 

\begin{definition}[Run Tree]
	Given a word $u \in \Sigma^\omega$, a run tree of an APA $\calA = (Q, q_0, \delta, c)$ on $u$ is a $Q$-labeled tree $(T, \ell)$ such that 
	\begin{itemize}
		\item $\ell(\epsilon) = q_0$, i.e., the root of the tree is labeled with the initial state of $\calA$, and
		\item For every $\tau \in T$ with $\ell(\tau) = q$, 
		\begin{align*}
			\big\{ \ell(\tau') \mid \tau' \in \mathit{children}_T(\tau) \big\} \models \delta\big(q, u(|\tau|)\big),
		\end{align*}
		i.e., for every node in the tree (the label of) all its children satisfy the boolean formula over states given by $\delta$.
	\end{itemize}
	The run-tree $(T, \ell)$ is \emph{accepting} if, for every infinite path in the tree, the minimal color that occurs infinitely many times (as given by $c$) is even. 
\end{definition}

We define $\calL(\calA) \subseteq \Sigma^\omega$ as all infinite words on which $\calA$ has an accepting run tree. 
We refer the reader to \citet{Vardi95} for more details on APAs. 

\begin{example}
	We consider the APA $\calA$ from \Cref{ex:apa}.
	For the readers' convenience, we depict it again in \Cref{fig:apa-re}.
	We want to show that $b^\omega \in \calL(\calA)$. 
	In \Cref{fig:apa-run}, we depict a possible accepting run tree of the automaton on $u = b^\omega$. 
	The tree starts in $q_0$. 
	Upon reading letter $b$, the universal branching forces us to restart one run from $q_0$ and one run from either $q_1$ or $q_2$.
	In our tree, we choose $q_2$ for the disjunctive choice, leading to the two children labeled with $q_2$ and $q_0$. 
	From state $q_2$, we then transition to state $q_3$ in the next step and stay there. 
	We can repeat this on every level, i.e., always branch from $q_0$ into both $q_2$ and $q_0$. 
	Note that on each infinite branch of the tree, we visit color $1$ at most once; the smallest color that occurs infinitely many times is thus $0$ on all branches; the run tree is accepting, showing that $b^\omega \in \calL(\calA)$.
\end{example}

Note that in case $\calA$ is non-deterministic -- i.e., every transition formula consists of a disjunction of states -- run trees can always have the form of sequences (i.e., trees where every node has a single child).

\begin{figure}[!t]
	
	\begin{subfigure}{\linewidth}
		\begin{center}
		\begin{tikzpicture}
		
		\node[rectangle, draw, thick, label={[yshift=-1pt]above:{\small \textbf{0}}}] at (-0.7,0) (n0) {$q_0$};
		
		\node[circle, draw, inner sep=1pt] at (1,0) (d0) {\tiny$\land$};
		
		\coordinate[] (c0) at (1, -0.5);
		
		\coordinate[] (c1) at (-0.7, -0.5);
		
		\node[circle, draw, inner sep=1pt] at (2,0) (d1) {\tiny$\lor$};

		\node[rectangle, draw, thick, label={[yshift=1pt]below:{\small \textbf{1}}}] at (3,0.7) (n1) {$q_1$};
		
		\node[rectangle, draw, thick, label={[yshift=-1pt]above:{\small \textbf{1}}}] at (3,-0.7) (n2) {$q_2$};
		
		\node[rectangle, draw, thick,label={[yshift=-1pt]above:{\small \textbf{0}}}] at (4.5,0) (n3) {$q_3$};
		
		\draw[->, thick] (-1.2, 0) to (n0);

		\draw[->, thick] (n0) to node[above,align=center] {\small $a, b, c$} (d0);
		
		\draw[->, thick] (d0) to (d1);
		
		\draw[-, thick] (d0) to (c0);
		\draw[-, thick] (c0) to (c1);
		\draw[->, thick] (c1) to (n0);
		
		\draw[->, thick] (d1) to (n1);
		\draw[->, thick] (d1) to (n2);

		\draw[->, thick] (n1)  to node[above,align=center, sloped] {\small $a$} (n3);
		
		\draw[->, thick] (n2) to node[below,align=center, sloped] {\small $b$} (n3);

		\draw[->, thick, loop left] (n1) to node[left,align=center] {\small $b, c$} (n1);
		\draw[->, thick, loop left] (n2) to node[left,align=center] {\small $a, c$} (n2);

		\draw[->, thick, loop right] (n3) to node[right,align=center] {\small $a, b, c$} (n3);
		
	\end{tikzpicture}
		\end{center}
		
		\subcaption{}\label{fig:apa-re}
	\end{subfigure}\\[5mm]
	\begin{subfigure}{\linewidth}
		
		\begin{center}
			\begin{tikzpicture}
				\node[circle, draw, thick,label=left:{\small \textbf{0}}] at (0, 0) (n0) {$q_0$};
				
				\node[circle, draw, thick,label=left:{\small \textbf{1}}] at (-2, -1.5) (n10) {$q_2$};
				\node[circle, draw, thick,label=left:{\small \textbf{0}}] at (2, -1.5) (n11) {$q_0$};
				
				\node[circle, draw, thick,label=left:{\small \textbf{0}}] at (-2, -3) (n20) {$q_3$};
				\node[circle, draw, thick,label=left:{\small \textbf{1}}] at (0.5, -3) (n21) {$q_2$};
				\node[circle, draw, thick,label=left:{\small \textbf{0}}] at (3.5, -3) (n22) {$q_0$};
				
				\node[circle, draw, thick,label=left:{\small \textbf{0}}] at (-2, -4.5) (n30) {$q_3$};
				\node[circle, draw, thick,label=left:{\small \textbf{0}}] at (0.5, -4.5) (n31) {$q_3$};
				
				\node[circle, draw, thick,label=left:{\small \textbf{1}}] at (2.5, -4.5) (n32) {$q_2$};
				\node[circle, draw, thick,label=left:{\small \textbf{0}}] at (4.5, -4.5) (n33) {$q_0$};
				
				\draw[->, thick] (n0) to (n10);
				\draw[->, thick] (n0) to (n11);
				
				\draw[->, thick] (n10) to (n20);
				\draw[->, thick] (n11) to (n21);
				\draw[->, thick] (n11) to (n22);
				
				\draw[->, thick] (n20) to (n30);
				\draw[->, thick] (n21) to (n31);
				
				\draw[->, thick] (n22) to (n32);
				\draw[->, thick] (n22) to (n33);
				
				\draw[->, thick,dashed] (n30) to (-2, -5.5);
				\draw[->, thick,dashed] (n31) to (0.5, -5.5);
				\draw[->, thick,dashed] (n32) to (2.5, -5.5);
				
				\draw[->, thick,dashed] (n33) to (4, -5.5);
				\draw[->, thick,dashed] (n33) to (5, -5.5);
			\end{tikzpicture}
		\end{center}
		
		\subcaption{}\label{fig:apa-run}
	\end{subfigure}

	\caption{In \Cref{fig:apa-re}, we depict the APA from \Cref{ex:apa}. In \Cref{fig:apa-run}, we sketch a run tree of this APA on the infinite word $b^\omega$.}
\end{figure}

\section{Correctness Proof}\label{app:proof-main}

In this section, we prove \Cref{theo:construction} (in \Cref{sec:sub:correct}) and \Cref{prop:correctProp} (in \Cref{sec:sub:mc}).

\subsection{Proof of \Cref{theo:construction}}\label{sec:sub:correct}

\main*

We focus here on proving the case where $\varphi = \llangle A \rrangle_\xi \, \pi\ldot \varphi'$. 
The proof for $\varphi = \llbracket A \rrbracket_\xi \, \pi\ldot \varphi'$ is analogous.

We use the construction of $\calA_{{\varphi}}$ from the proof sketch in the main body, which we re-iterate for the readers' convenience.
We assume that $\calA_{{\varphi'}}$ is $(\calG, \dot{s})$-equivalent to $\varphi'$ and let $\calA^\mathit{det}_{\varphi'} = (Q, q_0, \delta, c)$ be a DPA equivalent to $\calA_{\varphi'}$ (obtained via \Cref{prop:toDPA}).
We define $\calA_\varphi$ as $\calA_\varphi = (Q \times S, (q_0, \dot{s}), \delta', c')$ where $c'(q, s) := c(q)$ and, for each $l \in V \to S$, we define $\delta'\big((q, s), l\big)$ is defined as 
\begin{align*}
	\bigvee_{\substack{\bA : A \to \moves \\\forall \agent{i}, \agent{j} \in A\ldot (\agent{i}, \agent{j}) \in \xi  \\\Rightarrow\bA(\agent{i}) = \bA(\agent{j})}} \;\; &\bigwedge_{\substack{\bA' : \overline{A} \to \moves \\\forall \agent{i}, \agent{j} \in \overline{A}\ldot (\agent{i}, \agent{j}) \in \xi  \\\Rightarrow\bA'(\agent{i}) = \bA'(\agent{j})}} \big(\delta(q, l[\pi \mapsto s]), \kappa(s, \bA \oplus \bA')\big).
\end{align*}
We now formally prove that this construction fulfills the requirements of \Cref{theo:construction}, i.e., $\calA_\varphi$ is $(\calG, \dot{s})$-equivalent to $\varphi$.
That is, for every path assignment $\Pi : V \to S^\omega$, we have $\mathit{zip}(\Pi) \in \calL(\calA_\varphi)$ iff $\dot{s}, \Pi \models_\calG \varphi$. 
We show both directions of this equivalence as separate lemmas. 

\begin{lemma}\label{lem:dir1}
	For any $\Pi : V \to S^\omega$, if $\mathit{zip}(\Pi) \in \calL(\calA_\varphi)$ then $\dot{s}, \Pi \models_{\calG} \varphi$.
\end{lemma}
\begin{proof}
	Let $(T, \ell)$ be an accepting run of $\calA_\varphi$ on $\mathit{zip}(\Pi)$.
	We use the disjunctive choices made in $(T, \ell)$ to construct a strategy vector $\bF \in \mathit{shr}_\calG(A, \xi)$ that serves as a witness for the existential quantifier in the semantics of $\dot{s}, \Pi \models_{\calG} \varphi$.
	For each finite play $u \in S^+$,  we define $\bF(\agent{i})(u)$ for all $\agent{i} \in A$ as follows.
	We check if there exists a node $\tau$ in $(T, \ell)$ such that the nodes along $\tau$ are labeled by $u$, i.e., 
	\begin{align*}
		&\ell(\epsilon), \ell(\tau[0, 0]), \ell(\tau[0, 1]), \ldots, \ell(\tau[0, |\tau|-1]) = \\
		&\quad\quad (\_, u(0)),(\_, u(1)), \ldots, (\_, u(|u|-1)),
	\end{align*}
	where we write ``$\_$'' as we do not care about the automaton states. 
	If no such node exists, we define $\bF(\agent{i})(u)$ arbitrarily for all $\agent{i} \in A$ (any play that is compatible with the strategy never reaches this situation). 
	Otherwise, let $\ell(\tau) = \ell(\tau[0, |\tau|-1])= (q, s)$ where $q \in Q$. 
	Note that $s = u(|u|-1)$.
	By construction of $\calA_\varphi$, we have that the children of $\tau$ in $(T, \ell)$ satisfy $\delta'\big((q, s), \mathit{zip}(\Pi)(|\tau|)\big)$, i.e.,
	\begin{align*}
		\bigvee_{\substack{\bA : A \to \moves \\\forall \agent{i}, \agent{j} \in A\ldot (\agent{i}, \agent{j}) \in \xi  \\\Rightarrow\bA(\agent{i}) = \bA(\agent{j})}} \;\; &\bigwedge_{\substack{\bA' : \overline{A} \to \moves \\\forall \agent{i}, \agent{j} \in \overline{A}\ldot (\agent{i}, \agent{j}) \in \xi  \\\Rightarrow\bA'(\agent{i}) = \bA'(\agent{j})}} \\
		&\big(\delta(q, \mathit{zip}(\Pi)(|\tau| )[\pi \mapsto s]), \kappa(s, \bA \oplus \bA')\big).
	\end{align*}
	There must thus exist (at least one) $\bA : A \to \moves$ that satisfies $\xi$ such that for every $\bA' : \overline{A} \to \moves$ (also subject to $\xi$), there is a child of $\tau$ labeled with 
	\begin{align*}
		\big(\delta(q, \mathit{zip}(\Pi)(|\tau|)[\pi \mapsto s]), \kappa(s, \bA \oplus \bA')\big).
	\end{align*}
	Note that any such $\bA$ assigns agents that should share a strategy the same action.
	We pick any $\bA : A \to \moves$ that satisfies the disjunction and define
	\begin{align*}
		\bF(\agent{i})(u) := \bA(\agent{i})
	\end{align*}
	for each $\agent{i} \in A$.
	
	We claim that the strategy vector $\bF : A \to \strats{\calG}$ we have just constructed is winning for $A$. 
	By construction, it is easy to see that we assign the same strategy for agents that are required to share a strategy, i.e., $\bF \in \mathit{shr}_\calG(A, \xi)$.
	So take any vector $\bF' \in \mathit{shr}_\calG(\overline{A}, \xi)$ and let  $p :=  \play_\calG(s, \bF \oplus \bF')$. 
	We claim that $\dot{s}, \Pi[\pi \mapsto p] \models_\calG \varphi'$. 
	Note that by the \HyperATLSS{} semantics, this would imply that $\dot{s}, \Pi \models_\calG \varphi$ as required. 
	
	In the construction of $\bF$, we -- for each node $\tau$ -- always picked an action vector that corresponds to a disjunction that is satisfied $\tau$'s children in $(T, \ell)$.
	So, for any possible actions for $\overline{A}$ chosen by $\bF'$, the successor state is again a node in $(T, \ell)$ (by construction of $\delta'$). 
	There thus exists a path in $(T, \ell)$ here the state component equals $p$, i.e., a path that is labeled with
	\begin{align*}
		(q_0, p(0)) (q_1, p(1)) (q_2, p(2)) \cdots
	\end{align*}
	for some sequence of automaton states $q_0 q_1q_2 \cdots$.
	By definition of $\delta'$, the sequence of automaton state $q_0q_1q_2 \cdots$ (where $q_0$ is the initial state of $\calA^\mathit{det}_{\varphi'}$) is the unique run of $\calA^\mathit{det}_{\varphi'}$ on $\mathit{zip}(\Pi[\pi \mapsto p])$.
	As, in $(T, \ell)$, the above infinite tree-path is accepting, the sequence of automaton states is accepting in $\calA^\mathit{det}_{\varphi'}$ (in $\calA_\varphi$ we use the same state color as in $\calA^{\mathit{det}}_{{\varphi'}}$).
	We therefore get that $\mathit{zip}(\Pi[\pi \mapsto p]) \in \calL(\calA^\mathit{det}_{\varphi'}) = \calL(\calA_{\varphi'})$.
	By the assumption that $\calA_{\varphi'}$ is $(\calG, \dot{s})$-equivalent to $\varphi'$, we thus get that $\dot{s}, \Pi[\pi \mapsto p] \models_\calG \varphi'$.
	As this holds for all $\bF' \in \mathit{shr}_\calG(\overline{A}, \xi)$, we get $\bF$ is a witness for the existentially quantified strategies for $A$ and so  $\dot{s}, \Pi \models_\calG \varphi$ as required.
\end{proof}

\begin{lemma}\label{lem:dir2}
	For any $\Pi : V \to S^\omega$, if $\dot{s}, \Pi \models_{\calG} \varphi$ then $\mathit{zip}(\Pi) \in \calL(\calA_\varphi)$.
\end{lemma}
\begin{proof}
	We assume $\dot{s}, \Pi \models_{\calG} \varphi$ so, by the \HyperATLSS{} semantics, there exists a concrete witness strategy vector $\bF \in \mathit{shr}_\calG(A, \xi)$.
	The concrete vector $\bF$ satisfies that for all $\bF' \in \mathit{shr}_\calG(\overline{A}, \xi)$ we have that $\dot{s}, \Pi[\pi \mapsto \play_\calG(s, \bF \oplus \bF')] \models_\calG \varphi'$. 

	We use $\bF$ to construct an accepting run $(T, \ell)$ of $\calA_\varphi$ on $\mathit{zip}(\Pi)$. 
	We construct this infinite tree incrementally by adding children to existing nodes.
	Initially, we start with the root node $\epsilon$ and define $\ell(\epsilon) := (q_0, \dot{s})$ (where $q_0$ is the initial state of $\calA_{{\varphi}}^\mathit{det}$).
	(Note that $(q_0, \dot{s})$ is the initial state of $\calA_\varphi$).
	Now let $\tau \in T$ be any node in the tree constructed so far, and let
	\begin{align*}
		&\ell(\epsilon),\ell(\tau[0,0]),\ell(\tau[0, 1]), \ldots, \ell(\tau[0, |\tau|-1]) = \\
		&\quad\quad\quad (q_0, \dot{s}),(q_1, s_1), \ldots, (q_{|\tau|}, s_{|\tau|})
	\end{align*}
	be the label of the nodes along path $\tau$ (note that there are $|\tau| + 1$ nodes along $\tau$).
	We define a partial move vector $\bA: A \to \moves$ via $\bA (\agent{i}) := \bF(\agent{i})(\dot{s},s_1, \ldots, s_{|\tau|})$ for each $\agent{i} \in A$.
	That is, we let each agent  $\agent{i}$ in $A$ play the action that strategy $\bF(i)$ would fix on the finite path of states reached along $\tau$.
	Note that as $\bF \in \mathit{shr}_\calG(A, \xi)$, $\bA$ satisfies the sharing constraints, i.e., $\forall \agent{i}, \agent{j} \in A\ldot (\agent{i}, \agent{j}) \in \xi  \Rightarrow\bA(\agent{i}) = \bA(\agent{j})$.
	
	After we have fixed $\bA$, we consider all possible move vector $\bA' : \overline{A} \to \moves$ that satisfy the sharing constraint (i.e., $\forall \agent{i}, \agent{j} \in \overline{A}\ldot (\agent{i}, \agent{j}) \in \xi  \Rightarrow\bA'(\agent{i}) = \bA'(\agent{j}$).
	For each such $\bA'$, we add a new child of $\tau$ labeled with 
	\begin{align*}
		\Big(\delta\big(q_{|\tau|}, \mathit{zip}(\Pi)(|\tau|)[\pi \mapsto s_{|\tau|}]\big), \kappa\big(s_{|\tau|}, \bA \oplus \bA'\big)\Big)
	\end{align*}
	By construction (as we consider all possible $\bA'$), we get that the children of $\tau$ satisfy $\delta'\big((q_{|\tau|}, s_{|\tau|}), \mathit{zip}(\Pi)(|\tau|)\big)$.
	The constructed tree $(T, \ell)$ is thus a run tree of $\calA_{{\varphi}}$ on $\mathit{zip}(\Pi)$.
	
	We now claim that $(T, \ell)$ is accepting. Consider any infinite path in $(T, \ell)$ labeled by $(q_0, \dot{s}) (q_1, s_1) (q_2, s_2) \cdots$. 
	By construction of $(T, \ell)$, it is easy to see that there exists some $\bF' \in \mathit{shr}_\calG(\overline{A}, \xi)$ such that $\play_\calG(s, \bF \oplus \bF') = s_0s_1s_2 \cdots$ (we can let $\bF'$ always pick the action vector that has led to the node being added to the tree during the construction).
	
	By the assumption on $\bF$, we thus get that $\dot{s}, \Pi[\pi \mapsto s_0s_1s_2 \cdots] \models_\calG \varphi'$.
	By the assumption that $\calA_{\varphi'}$ is $(\calG, \dot{s})$-equivalent to $\varphi'$, we thus get that $\mathit{zip}(\Pi[\pi \mapsto s_0s_1s_2 \cdots]) \in \calL(\calA_{{\varphi'}}) = \calL(\calA_{{\varphi'}}^\mathit{det})$.  
	By construction of $\calA_\varphi$, the automaton sequence $q_0 q_1 q_2\cdots$ is the unique run of $\calA_{\varphi'}^\mathit{det}$ on $\mathit{zip}(\Pi[\pi \mapsto s_0s_1s_2 \cdots])$ and therefore accepting (i.e., the minimal color that occurs infinitely often is even).
	As this holds for all paths in $(T, \ell)$, we have constructed an accepting run tree of $\calA_\varphi$ on $\mathit{zip}(\Pi)$, so $\mathit{zip}(\Pi) \in \calL(\calA_{{\varphi}})$ as required.
\end{proof}

\Cref{lem:dir1} and \Cref{lem:dir2} conclude the proof of \Cref{theo:construction}.

\subsection{Proof of \Cref{prop:correctProp}}\label{sec:sub:mc}

\correctProp*
\begin{proof}
	It is easy to see that the preprocessing done in lines \ref{line:extract-nested-formula}-\ref{line:extend-label} modifies $\calG$ and $\psi$ such that the satisfaction is preserved:
	By a simple inductive argument, the set $S_{\varphi'}$ computed in \cref{line:rec} thus contains all states $s$ such that $s, \emptyset \models_\calG {\varphi'}$ and extending the label of those states with a fresh proposition $p_{\varphi'}$ does not change this. 
	As $p_{\varphi'}$ now holds in exactly those states where $s, \emptyset \models_\calG {\varphi'}$, we can -- in $\psi$ -- soundly replace $\varphi'_\pi$ with $(p_{\varphi'})_\pi$.
	
	The more interesting case is to argue that the loop is correct. 
	We maintain a simple loop invariant (LI): 
	
	\begin{center}
		\emph{After \cref{line:product} is executed, $\calA$ is $(\calG, \dot{s})$-equivalent to $\lQ A_j \rQ_{\xi_j} \pi_j \ldots \lQ A_n \rQ_{\xi_n} \pi_n\ldot \psi$.}
	\end{center}

	\noindent
	It is easy to see that this loop invariant holds initially:
	Initially, the APA $\calA$ we construct from path formula $\psi$ (in \cref{line:ltlToApa}) is $(\calG, \dot{s})$-equivalent to the path formula $\psi$ (in fact, after line \ref{line:ltlToApa}, $\calA$ is $(\calG, s)$-equivalent to $\psi$ for every $s \in S$).
	In the first iteration of the loop (where $j=n$) we now apply \lstinline[style=default, language=custom-lang]|product| to quantifier $\lQ A_n \rQ_{\xi_n} \pi_n$, so by \Cref{theo:construction} (describing the construction in \lstinline[style=default, language=custom-lang]|product|) the APA $\calA$ is -- after \cref{line:product} -- $(\calG, \dot{s})$-equivalent to $\lQ A_n \rQ_{\xi_n} \pi_n\ldot \psi$ as required by the LI.
	
	For the inductive case, we can assume that  -- at the beginning of the loop, before \cref{line:product} -- $\calA$ was $(\calG, \dot{s})$-equivalent to $\lQ A_{j+1} \rQ_{\xi_{j+1}} \pi_{j+1} \ldots \lQ A_n \rQ_{\xi_n} \pi_n\ldot \psi$ (from the previous iteration). 
	As before, we apply \lstinline[style=default, language=custom-lang]|product| to $\lQ A_j \rQ_{\xi_j} \pi_j$, so by \Cref{theo:construction} -- after \cref{line:product} -- $\calA$ is $(\calG, \dot{s})$-equivalent to $\lQ A_j \rQ_{\xi_j} \pi_j \ldots \lQ A_n \rQ_{\xi_n} \pi_n\ldot \psi$ as required by the LI.
	
	After the loop, $\calA$ is thus $(\calG, \dot{s})$-equivalent to the entire formula $\varphi$. 
	By definition of $(\calG, \dot{s})$-equivalence, this means that $\dot{s}, \emptyset \models_\calG \varphi$ iff $\mathit{zip}(\emptyset) \in \calL(\calA)$.
	In \cref{line:emptinessCheck}, we thus add $\dot{s}$ to $\mathit{Sol}$ iff $\dot{s}, \emptyset \models_\calG \varphi$, so \lstinline[style=default, language=custom-lang]|modelCheck($\calG$,$\varphi$)| returns $\mathit{Sol} = \{s \in S \mid s, \emptyset \models_\calG \varphi \}$ as required.
\end{proof}

\section{Model-Checking Complexity}

In this section, we formally analyze the complexity of \Cref{alg:main-alg}.
Our main complexity measure is the \emph{nesting-rank} of a formula.

\begin{definition}[Rank]
	We assign each \HyperATLSS{} path and state formula a nesting-rank.
	For path formulas we define
	\begin{align*}
		\mathit{rank}(a_\pi) &:= 0 \\
		\mathit{rank}(\psi_1 \land \psi_2) &:= \max\big( \mathit{rank}(\psi_1), \mathit{rank}(\psi_2) \big)\\
		\mathit{rank}(\neg \psi) &:=  \mathit{rank}(\psi)\\
		\mathit{rank}(\ltlN \psi) &:=  \mathit{rank}(\psi)\\
		\mathit{rank}(\psi_1 \ltlU \psi_2) &:= \max\big( \mathit{rank}(\psi_1), \mathit{rank}(\psi_2) \big)\\
		\mathit{rank}(\varphi_\pi) &:=  \mathit{rank}(\varphi)
	\end{align*}
	and for state formulas we define
	\begin{align*}
		\mathit{rank}(\psi) &:= \mathit{rank}(\psi) \\
		\mathit{rank}(\lQ A \rQ_\xi \, \pi\ldot \varphi) &:= \mathit{rank}(\varphi) + 1.
	\end{align*}
\end{definition}

Intuitively, the rank gives the maximal number of quantifiers in any nested state formula.
When executing \lstinline[style=default, language=custom-lang]|modelCheck($\calG$,$\varphi$)|, we recursively call \lstinline[style=default, language=custom-lang]|modelCheck| on state formulas of the form $\lQ A_1 \rQ_{\xi_1} \pi_1 \ldots \lQ A_n \rQ_{\xi_n} \pi_n\ldot \psi$. 
By the definition $\mathit{rank}$, each such call satisfies $n \leq \mathit{rank}(\varphi)$; the rank thus gives an upper bound on \emph{consecutive} applications of  \lstinline[style=default, language=custom-lang]|product|. 
In all recursive calls, the size of $\calA$ is thus at most $\big(2 \cdot \mathit{rank}(\varphi)\big)$-exponential, giving us an upper bound on the model-checking complexity.

\begin{proposition}
	Model-checking for a \HyperATLSS{} formula with nesting-rank $m$ is in $2m$-\EXPTIME{}.
\end{proposition}

If we consider arbitrary formulas where we do not parameterize the complexity by their rank, MC is non-elementary.

\begin{proposition}
	The model-checking problem for \HyperATLSS{} is decidable in non-elementary time. 
\end{proposition}

\paragraph{Lower Bounds.}

It is easy to see that \HyperATLSS{} subsumes \HyperATLS{}.
From the known \HyperATLS{} lower bounds \cite{BeutnerF23} we thus get:

\begin{proposition}[\citet{BeutnerF23}]\label{prop:lb}
	The model-checking problem for \HyperATLSS{} is non-elementary-hard. 
\end{proposition}

\fi

\end{document}